\documentclass[sigconf]{acmart}
\usepackage{microtype}
\usepackage{graphicx}
\usepackage{subfigure}
\usepackage{booktabs} 
\usepackage{hyperref}

\usepackage[ruled,norelsize]{algorithm2e}
\usepackage{algorithmic}
\usepackage{amsmath}
\usepackage{mathtools}
\usepackage{amsthm}
\usepackage{booktabs}
\usepackage{xspace}
\usepackage{tabularx}

\usepackage[utf8]{inputenc} 
\usepackage[T1]{fontenc}    

\usepackage{url}            
\usepackage{booktabs}       
\usepackage{amsfonts}       
\usepackage[subfigure]{tocloft}
\usepackage{nicefrac}       
\usepackage{microtype}      
\usepackage{xcolor}         

\usepackage{graphicx}      
\usepackage[toc,page,header]{appendix}
\usepackage{minitoc}
\usepackage{makecell}
\usepackage{adjustbox}
\usepackage{colortbl}
\definecolor{Gray}{gray}{0.92}
\usepackage{amsmath,bm,bbm}
\usepackage{mathtools}
\usepackage{amsthm}
\usepackage{multirow}
\usepackage{wrapfig}
\usepackage{enumitem}

\usepackage[capitalize,noabbrev]{cleveref}

\usepackage{mdframed}

\theoremstyle{plain}
\definecolor{theoremcolor}{rgb}{0.94, 0.94, 0.94}
\definecolor{examplecolor}{rgb}{1, 1, 1.0}
\mdfsetup{
    backgroundcolor=theoremcolor,
    linewidth=0pt,
}

\newcommand{\myref}[1]{Eq.~\ref{#1}}

\usepackage[capitalize]{cleveref}
\crefname{section}{Sec.}{Secs.}
\Crefname{section}{Section}{Sections}
\Crefname{table}{Table}{Tables}
\crefname{table}{Tab.}{Tabs.}

\newcommand{\ie}[0]{i.e., }
\newcommand{\eg}[0]{e.g., }

\newcommand{\D}{\mathcal{D}}
\newcommand{\T}{\mathcal{T}}
\newcommand{\Z}{\mathcal{Z}}
\newcommand{\abbr}{DRM\xspace} 
\definecolor{Gray}{gray}{0.9}
\definecolor{Gray1}{gray}{0.8}
\definecolor{Gray2}{gray}{0.7}

\newmdtheoremenv[linewidth=0pt,innerleftmargin=4pt,innerrightmargin=4pt]{defn}{Definition}
\newmdtheoremenv[linewidth=0pt,innerleftmargin=4pt,innerrightmargin=4pt]{prop}{Proposition}
\newmdtheoremenv[linewidth=0pt,innerleftmargin=4pt,innerrightmargin=4pt]{assump}{Assumption}
\newmdtheoremenv[linewidth=0pt,innerleftmargin=0pt,innerrightmargin=0pt,backgroundcolor=examplecolor]{example}{Example}
\newmdtheoremenv{corollary}{Corollary}
\newmdtheoremenv[linewidth=0pt,innerleftmargin=4pt,innerrightmargin=4pt]{theo}{Theorem}
\newmdtheoremenv[linewidth=0pt,innerleftmargin=4pt,innerrightmargin=4pt]{lemma}{Lemma}

\AtBeginDocument{%
  \providecommand\BibTeX{{%
    \normalfont B\kern-0.5em{\scshape i\kern-0.25em b}\kern-0.8em\TeX}}}

\copyrightyear{2023} \acmYear{2023} \setcopyright{acmlicensed}\acmConference[KDD '23]{Proceedings of the 29thACM SIGKDD Conference on Knowledge Discovery and Data Mining}{August6--10, 2023}{Long Beach, CA, USA}\acmBooktitle{Proceedings of the 29th ACM SIGKDD Conference on KnowledgeDiscovery and Data Mining (KDD '23), August 6--10, 2023, Long Beach, CA,USA}\acmPrice{15.00}\acmDOI{10.1145/3580305.3599313}\acmISBN{979-8-4007-0103-0/23/08}

\begin{document}

\title{Domain-Specific Risk Minimization for Domain Generalization}

\author{Yi-Fan Zhang}
\authornote{YF Zhang is also affiliated with School of Artificial Intelligence, University of Chinese Academy of Sciences.}
\affiliation{%
  \institution{Institute of Automation}
  \country{Beijing, China}
}
 \email{yifanzhang.cs@gmail.com}
\author{Jindong Wang}
\authornote{Corresponding author: Jindong Wang.}
\affiliation{%
  \institution{Microsoft Research Asia}
   \country{Beijing, China}
   }
   \email{jindong.wang@microsoft.com}

\author{Jian Liang}
\affiliation{%
  \institution{Institute of Automation}
   \country{Beijing, China}}
\email{jian.liang@nlpr.ia.ac.cn}
\author{Zhang Zhang}
\affiliation{%
  \institution{Institute of Automation}
   \country{Beijing, China}}
\email{zzhang@nlpr.ia.ac.cn}
\author{Baosheng Yu}
\affiliation{%
  \institution{The University of Sydney}
   \country{Australia}}
\email{bayu0826@uni.sydney.edu.au}
\author{Liang Wang}
\affiliation{%
  \institution{Institute of Automation}
   \country{Beijing, China}}
\email{wangliang@nlpr.ia.ac.cn}
\author{Dacheng Tao}
\affiliation{%
  \institution{The University of Sydney}
   \country{Australia}}
\email{dacheng.tao@sydney.edu.au}
\author{Xing Xie}
\affiliation{%
  \institution{Microsoft Research Asia}
   \country{Beijing, China}}
\email{xingx@microsoft.com}







\renewcommand{\shortauthors}{Zhang, et al.}

\begin{abstract}
  Domain generalization (DG) approaches typically use the hypothesis learned on source domains for inference on the unseen target domain. However, such a hypothesis can be arbitrarily far from the optimal one for the target domain, induced by a gap termed ``adaptivity gap''. Without exploiting the domain information from the unseen test samples, adaptivity gap estimation and minimization are intractable, which hinders us to robustify a model to any unknown distribution. In this paper, we first establish a generalization bound that explicitly considers the adaptivity gap. Our bound motivates two strategies to reduce the gap: the first one is ensembling multiple classifiers to enrich the hypothesis space, then we propose effective gap estimation methods for guiding the selection of a better hypothesis for the target. The other method is minimizing the gap directly by adapting model parameters using online target samples. We thus propose \textbf{Domain-specific Risk Minimization (\abbr)}. During training, \abbr models the distributions of different source domains separately; for inference, \abbr performs online model steering using the source hypothesis for each arriving target sample. Extensive experiments demonstrate the effectiveness of the proposed \abbr for domain generalization. Code is available at: \url{https://github.com/yfzhang114/AdaNPC}.
\end{abstract}

\ccsdesc[500]{Computing methodologies~Transfer learning}
\ccsdesc[500]{Computing methodologies~Learning latent representations}
\ccsdesc[500]{Computing methodologies~Neural networks}

\keywords{Domain Generalization, Test-time Adaptation, Adaptivity gap}

\maketitle

\section{Introduction}

Machine learning models generally suffer from degraded performance when the training and test data are non-IID (independently and identically distributed).
To overcome the brittleness of classical empirical risk minimization (ERM), there is an emerging trend of developing out-of-distribution (OOD) generalization approaches~\cite{MuaBalSch13,li2018domain}, where models trained on multiple source domains/datasets can be directly deployed on \emph{unseen} target domains.
Various OOD frameworks are proposed, e.g., disentanglement~\cite{zhang2021towards,oh2022learning}, causal invariance~\cite{arjovsky2020invariant,liu2021learning,zhang2022exploring}, and adversarial training~\cite{ganin2016domain,zhang2023free,shi2022pairwise}.

Existing approaches might rely on two strong assumptions.
(i) \textbf{Hypothesis over-confidence.}
Most works directly apply a source-trained hypothesis to \emph{any} unseen target domains~\cite{arjovsky2020invariant,krueger2021out,rame2021fishr} by implicitly assuming that \emph{the training hypothesis space contains an ideal target hypothesis}.
However, the IID and OOD performances are not always positively correlated~\cite{teney2022id}, \ie the optimal hypothesis on source domains might not perform well on any target domains.
The distance between the optimal source and target hypothesis is termed \emph{adaptivity gap}~\cite{dubey2021adaptive}, which is even shown can be arbitrarily large~\cite{chu2022dna}.
(ii) \textbf{Pessimistic adaptivity gap reduction.} Although the adaptivity gap is ubiquitous, it is almost impossible to identify and minimize due to the unavailability of OOD target samples.
As a consequence, there exists no approach that can tackle \emph{all} kinds of distribution shifts at once (e.g., diversity shift in PACS~\cite{li2017deeper} and correlation shift in the Colored MNIST~\cite{arjovsky2020invariant}), but only a specific kind~\cite{ye2022ood}.
In a word, it is almost impossible to robustify a model to arbitrarily unknown distribution shift \emph{without} utilizing the target samples during inference.

To our best knowledge, the two disadvantages are always neglected by the commonly-used domain adaptation and generalization bounds \cite{ben2010theory,albuquerque2020adversarial,zhao2019learning}, which mostly ignore the terms that are related to the target domain.
To this end, we introduce a new generalization bound that independent on the choice of hypothesis space and explicitly considers the adaptivity gap between source and target.
The bound motivates two possible test-time adaptation strategies: the first one is to train specific classifiers for different source domains, and then dynamically ensemble them, which is shown able to enrich the set of the hypothesis space~\cite{domingos1997does}.
The other is to utilize the arriving target samples, namely once a target sample is given, we update the model by its provided target domain information. To summarize, this paper makes the following contributions:

1. \textbf{A novel perspective.} We provide a new generalization bound that does not depend on the choice of hypothesis space and explicitly considers the adaptivity gap between source domains and the target domain. Our bound is shown tighter than the existing one and provides intuition for reweighting methods, test-time adaptation methods, and classifier ensembling methods for good domain generalization performance.

2. \textbf{A new approach.} We propose \abbr method, which consists of two components: (i) During training, \abbr constructs specific classifiers for source domains and is trained by reweighting empirical loss. (ii) During the test, \abbr performs test-time model selection and retraining for each target sample. Thus, the source classifiers are dynamically changed for each target data and we can enrich the support set of the hypothesis space in this way to minimize the adaptivity gap directly.

3. \textbf{Extensive experiments.} We perform extensive experiments on popular OOD benchmarks showing that \abbr (1) achieves very competitive generalization performance on both diversity shift benchmarks and correlation shift benchmarks; (2) beats most existing test-time adaptation methods with a large margin; (3) is orthogonal to other DG methods; (4) reserves strong recognition capability on source domains, and (5) is parameter-efficient and converges even faster than ERM thanks to the structure.

\section{Related work}
\label{sec:related}

\textbf{Domain adaptation and domain generalization}
Domain/Out-of-distribution generalization \cite{wang2022generalizing,zhang2021learning,li2018deep,zhang2022generalizable,lu2022domain,lu2023out} aims to learn a model that can extrapolate well in unseen environments. Representative methods like Invariant Risk Minimization (IRM) \cite{arjovsky2020invariant} concentrate on the objective of extracting data representations that lead to invariant prediction across environments under a multi-environment setting. In this paper, we emphasize the importance of considering the adaptivity
gap and using online target data for adaptation. Without an invariance strategy, the proposed \abbr~can attain superior generalization capacity.

\textbf{Test-time adaptive methods}~\cite{liang2023comprehensive} are recently proposed to utilize target samples. Test-time Training methods need to design proxy tasks during tests such as self-consistence~\cite{zhang2021test}, rotation prediction~\cite{sun2020test} and need extra models; Test-time adaptation methods adjust model parameters based on unsupervised objectives such as entropy minimization~\cite{wang2020tent} or update a prototype for each class~\cite{iwasawa2021test}. Domain-adaptive method~\cite{dubey2021adaptive} needs extra models for adapting to the target domain. Non-Parametric Adaptation~\cite{zhang2023adanpc} needs to store all source domain instances. Our generalization bound indicates that these methods can explicitly reduce the target loss upper bound. In this paper, we propose other ways to perform test-time adaptation, \ie multi-classifier dynamic combination and retraining.

\textbf{Ensemble learning in domain generalization} learns ensembles of multiple specific models for different source domains to improve the generalization ability, \eg domain-specific backbones~\cite{ding2017deep}, domain-specific classifiers~\cite{wang2020dofe}, and domain-specific batch normalization~\cite{segu2020batch}. Domain-specific classifiers are also used in this work; however, empirical results show that directly ensembling multiple classifiers with a uniform weight degrades the performance, and the proposed \abbr achieves superior results in contrast. 

\textbf{Labeling function shift and multi classifiers}. Labeling function shift or correlation shift is not a novel concept and is commonly used in domain adaptation~\cite{zhao2019learning,stojanov2021domain,zhang2013domain} or domain generalization~\cite{ye2022ood}. There are also some studies on DG that are proposed to tackle this problem. CDANN\cite{li2018deep} considers the scenario where both $P(X)$ and $P(Y|X)$ change across domains and proposes to learn a conditional invariant neural network to minimize the discrepancy in $P(X|Y)$ between different domains.~\cite{liu2021domain} explores both the correlation and label shifts in DG and aligns the correlation shift via variational Bayesian inference. The proposed \abbr~is different from these studies because we want the labeling functions $P(Y|X)$ to be more specific to each domain than invariant. 

\section{A Bound by Considering Adaptivity Gap}\label{sec:theory}
\begin{figure}[b]
    \subfigure[]{
    \includegraphics[width=0.3\linewidth]{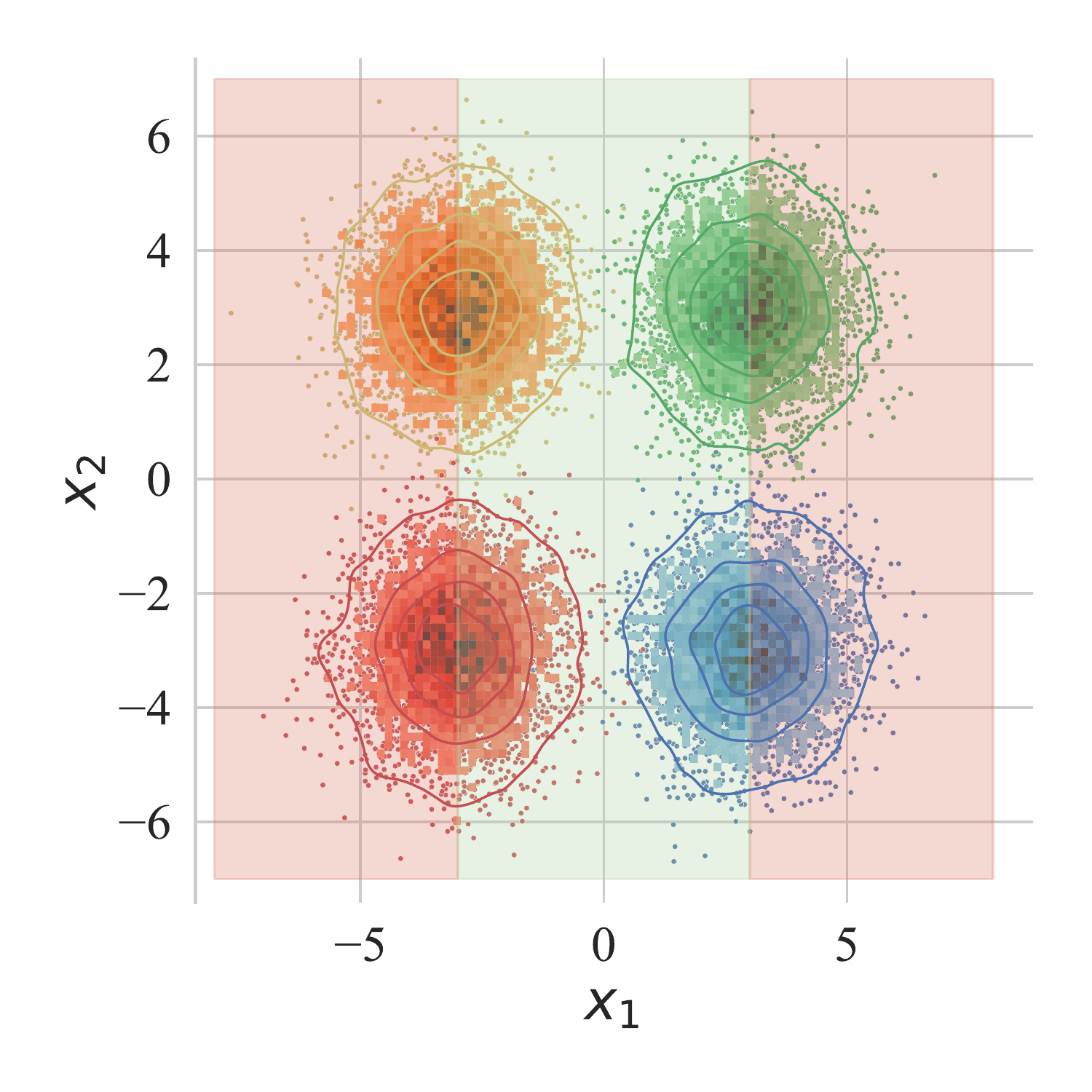}
    }
    \subfigure[]{
    \includegraphics[width=0.3\linewidth]{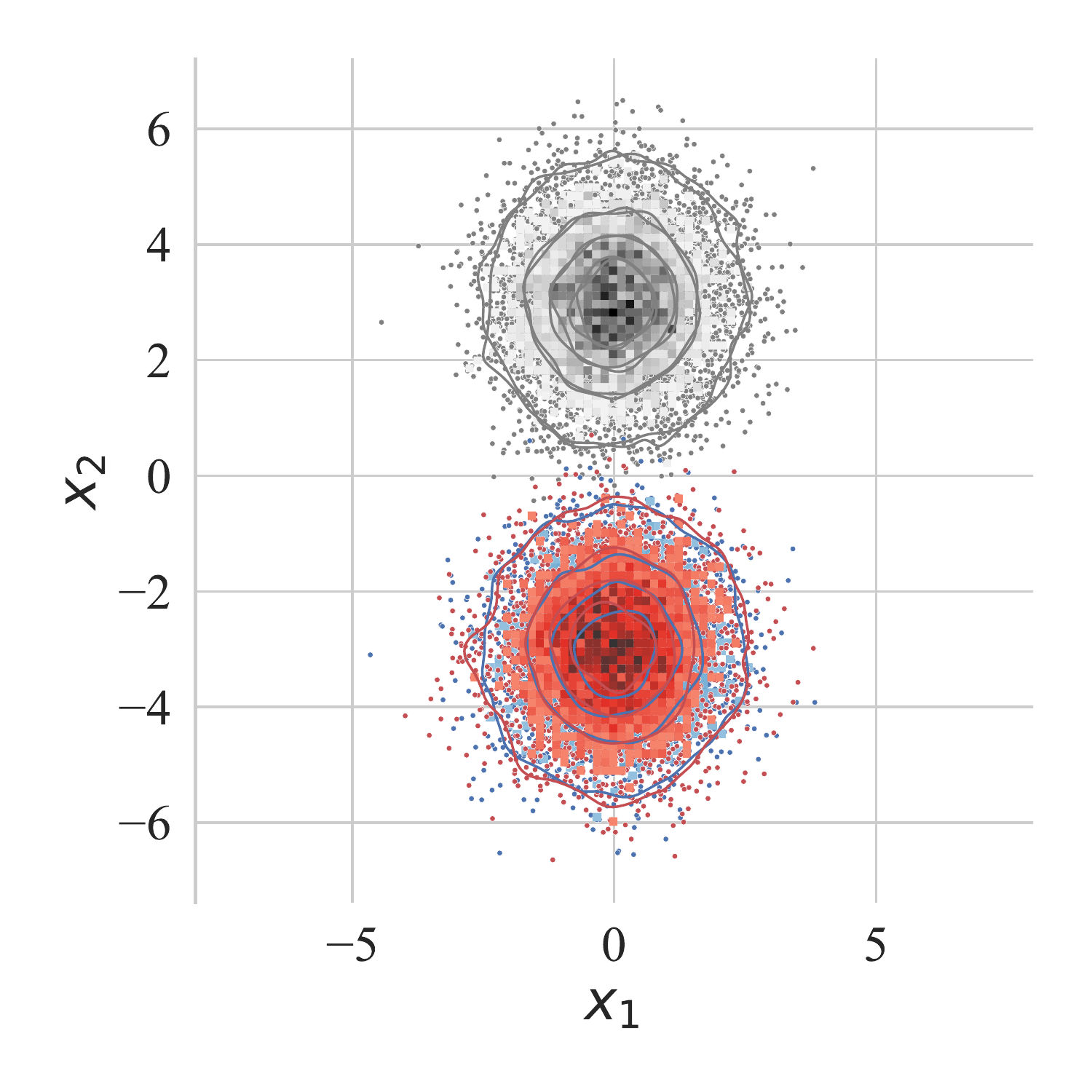}
    }
     \subfigure[]{
    \includegraphics[width=0.3\linewidth]{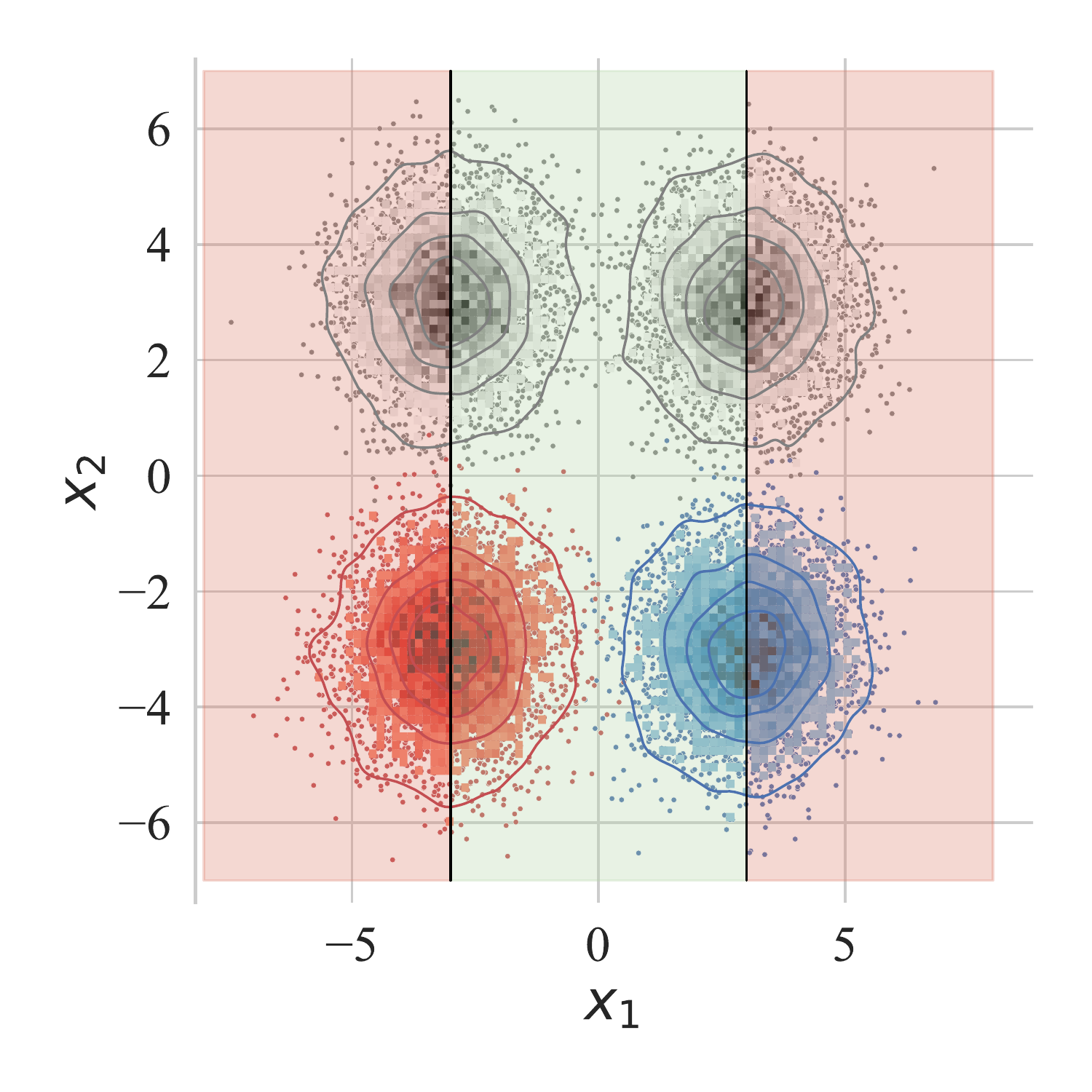}
    \label{fig:examples_drm}
    }
    \caption{A failure case of invariant representations for domain generalization. (a) Four domains in different colors: {\color{orange}orange} ($\mu_o=[-3.0, 3.0]$), {\color{green}green} ($\mu_g=[3.0,3.0]$), {\color{red}red} ($\mu_r=[-3.0,-3.0]$) and {\color{blue}blue} ($\mu_b=[3.0,-3.0]$).  (b) Invariant representations learnt from domain $\D_r$ and $\D_b$ by feature transformation $g(X)=\mathbb{I}_{x_1<0}\cdot(x_1+3)+\mathbb{I}_{x_1>0}\cdot(x_1-3)$. The {\color{gray}grey} color indicates the transformed target domains. (c) The classification boundary learned by \abbr.}
    \label{fig:examples_all}\label{fig:examples}
\end{figure}
\textbf{Problem Formulation.} Let $\mathcal{X,Y,Z}$ denote the input, output, and feature space, respectively.
We use $X,Y,Z$ to denote the random variables taking values from $\mathcal{X,Y,Z}$, respectively.
We focus on the domain generalization setting, where a labeled training dataset  consisting of several different but related training distributions (domains) is given. Formally, $\mathcal{D} = \cup_{i=1}^K\mathcal{D}_i$, where $K$ is the number of domains.
Each $\mathcal{D}_i$ corresponds to a joint distribution $P_i(X,Y)$ with an optimal classifier $f_i:\mathcal{X} \to \{0,1\}$\footnote{Most theories and examples in this paper considers binary classification for easy understanding and can be easily extended to multi-class classification.}.
We assume the output $Y=f_i(X)$ is given by a classifier, $f_i$, which varies from domain to domain. We formally define the classification error, which will be used in our theoretical analysis.

\begin{defn}
\textbf{(classification error.)} Let $g:\mathcal{X} \to \mathcal{Z}$ and $h:\Z\to\{0,1\}$ denote the encoder/feature transformation and the prediction head, respectively.
The error incurred by hypothesis $\hat{f}:=h\circ g$ under domain $\mathcal{D}_i$ can be defined as $\epsilon_i(\hat{f})=\mathbb{E}_{X\sim\mathcal{D}_i}[|\hat{f}(X)-f_i(X)|]$. Given $f_i$ and $\hat{f}$ as binary classification functions, we have 
\begin{equation}
\begin{aligned}
\epsilon_i(\hat{f})=\epsilon_i(\hat{f},f_i)&=\mathbb{E}_{X\sim\mathcal{D}_i}\left[|\hat{f}(X)-f_i(X)|\right]\\
&=P_{X\sim\mathcal{D}_i}(\hat{f}(X)\neq f_i(X)).
\end{aligned}
\label{equ:loss}
\end{equation}
\end{defn}

In real applications, a source domain-trained model will be deployed to classify data samples in an online manner and we can adjust the model using unlabeled online instances~\cite{iwasawa2021test,wang2020tent}. Because the proposed method works fully online and has no requirement for offline unlabelled data, therefore can be compared fairly with existing DG methods~\cite{iwasawa2021test}.

\textbf{Existing analysis on OOD}
Existing popular approaches on OOD focus on learning invariant representations~\cite{li2018domain,ganin2016domain} with the following theoretical intuition.
\begin{prop}
\label{theo1:bound}
\textbf{(Informal)} Denote $\mathcal{\tilde{D}}_i$ as the induced distribution over feature space $\mathcal{Z}$ for every distribution $\mathcal{D}_i$ over raw space. Here we use $\mathcal{H}$ as a hypothesis space defined on feature space, \ie $\mathcal{H}\subseteq \{h:\mathcal{Z}\rightarrow\{0,1\}\}$. The following inequality holds for the risk $\epsilon_{\T}(\hat{f})$ on target domain $\mathcal{D}_\mathcal{T}$ (See appendix \ref{app:bound} for definition of $\mathcal{H}$-divergence $d_{\mathcal{H}}$ and formal derivations):
\begin{equation}
    \epsilon_\T(\hat{f}) \leq \mathcal{O}\left(\lambda_\alpha+{\sum_{i=1}^{K}\epsilon_{i}(\hat{f})}+\sum_{l=1}^{K}\sum_{k=1}^{K} d_{\mathcal{H}}\left(\tilde{\mathcal{D}_{l}},\tilde{\mathcal{D}_{k}}\right)\right),
    \label{lemma:bound}
\end{equation}
where $\lambda_\alpha$ is the optimal hypothesis that achieves the lowest risk under both target and source domains.
\end{prop}
where a feature transformation $g$ is learned such that the induced source distributions on $\mathcal{Z}$ are close to each other, and a prediction head $h$ over $\mathcal{Z}$ is to achieve small empirical errors on source domains.
The bound depends on the risk of the optimal hypothesis $\lambda_\alpha$, namely, we assume the hypothesis space contains an optimal classifier that performs well on both the source and the target. 




\textbf{Adaptivity gap.} The above assumption cannot be guaranteed to hold true under all scenarios and is usually intractable to compute for most practical hypothesis spaces, making the bound conservative and loose.
Besides, even if we have the optimal classifier, it is almost impossible to find the optimal one using given source domains. The reason is that the classifier trained by the average risk across domains can lie far from the optimal classifier for a target domain~\cite{dubey2021adaptive,chu2022dna}, induced by adaptivity gap:\footnote{The adaptivity gap is NOT the same as labeling functions difference~\cite{zhao2019learning}, where the latter measures the difference of two hypotheses: $\min\{{\color{brown}{\mathbb{E}_{\mathcal{D}_i}[|f_i-f_\T|]}},\mathbb{E}_{{\mathcal{D}}_\T}[|f_i-f_\T|]\}$. However, the error of target hypothesis $f_\T$ on the source domain is intractable to estimate and meaningless for DG~\cite{kpotufe2018marginal}. The definition of adaptivity gap directly measures if the source classifier performs well on the target.}
\begin{defn}[Adaptivity gap]
    The adaptivity gap between $\mathcal{D}_i$ and the target domain $\mathcal{D}_\T$ can be formally defined as $\mathbb{E}_{{\mathcal{D}}_\T}[|f_i-f_\T|]$, namely the error incurred by using $f_i$ for inference in  $\mathcal{D}_\T$.
\end{defn}

\textbf{A failure case of marginal invariant representation.} We construct a simple counterexample where invariant representations fail to generalize.
As shown in \cref{fig:examples}, given the following four domains: $\mathcal{D}_o\sim\mathcal{N}([-3,3],I), \mathcal{D}_g\sim\mathcal{N}([3,3],I), \mathcal{D}_r\sim\mathcal{N}([-3,-3],I),\mathcal{D}_b\sim\mathcal{N}([3,-3],I)$, where $X=(x_1,x_2)$ and
\begin{equation}
\begin{small}
\begin{aligned}
    &f_o(X)=\begin{cases}
0 & \text{if } x_1\leq -3 \\
1 & \text{otherwise}
\end{cases}, 
  f_r(X)=\begin{cases}
0 & \text{if } x_1\leq -3 \\
1 & \text{otherwise}
\end{cases},\\
 &  f_g(X)=\begin{cases}
1 & \text{if } x_1\leq 3 \\
0 & \text{otherwise}
\end{cases}, 
  f_b(X)=\begin{cases}
1 & \text{if } x_1\leq 3 \\
0 & \text{otherwise}
\end{cases},
\end{aligned}
\end{small}
\end{equation}
where $I$ indicates the identity matrix.
Then, the optimal hypothesis $f^*(X)=1\text{ iff } x_1\in(-3,3)$ achieves perfect classification on all domains\footnote{Although Gaussian distributions put some mass on parts of the input space where this $f^*$ misclassifies some examples ($x_1>3$ for $\D_r$), the density of these scopes are very small and can be ignored.}.
Let $\D_r,\D_b$ denote source domains and $\D_o,\D_g$ denote target domains. Given hypothesis $\hat{f}:=h\circ g$ where the feature transformation function is $g(X)=\mathbb{I}_{x_1<0}\cdot(x_1+3)+\mathbb{I}_{x_1>0}\cdot(x_1-3)$ in \cref{fig:examples} (b), namely, the invariant representation of $\D_r,\D_b$ is learnt, which is $\D_{rb}=g\circ\D_b=g\circ\D_r=\mathcal{N}([0,-3],I)$. However, the labeling functions $f_r$ of $\D_r$ and $f_b$ of $\D_b$ are just the reverse such that $f_r(X)=1-f_b(X);\forall X\in\D_{rb}$.
In this case, according to~\myref{equ:loss}, we have that $\epsilon_{rb}(\hat{f})$ is equal to: 
\begin{equation}
\begin{small}
\begin{aligned}
&=\epsilon_{r}(h\circ g)+\epsilon_b(h\circ g)\\
&=P_{X\sim g\circ\mathcal{D}_r}(h(X)\neq f_r(X))+P_{X\sim g\circ\mathcal{D}_b}(h(X)\neq f_b(X)) \\
&=1-P_{X\sim \mathcal{D}_{rb}}(h(X)\neq f_b(X))+P_{X\sim \mathcal{D}_{rb}}(h(X)\neq f_b(X))=1
\end{aligned}
\end{small}
\end{equation}

Therefore, the invariant representation leads to large joint errors on all source and target domains for any prediction head $h$ without considering the adaptivity gap.
Motivated by this, we provide a tighter OOD upper bound that considers the adaptivity gap.

\begin{prop}
Let $\{{\mathcal{D}_i},f_i\}_{i=1}^K$ and ${\mathcal{D}}_\T,f_\T$ be the empirical distributions and corresponding labeling function for source and target domain, respectively.
For any hypothesis $\hat{f}\in\mathcal{H}$, given mixed weights $\{\alpha_i\}_{i=1}^K;\sum_{i=1}^K\alpha_i=1,\alpha_i\geq0$, we have:
\begin{equation}
\begin{small}
\begin{aligned}
\epsilon_\T(\hat{f})\leq\sum_{i=1}^K\left(\mathbb{E}_{X\sim{\mathcal{D}}_i}\left[\alpha_i\frac{P_{\T}(X)}{P_{i}(X)}|\hat{f}-f_i|\right]+\alpha_i\mathbb{E}_{{\mathcal{D}}_\T}[|f_i-f_\T|]\right).
\label{theo:2bound}
\end{aligned}
\end{small}
\end{equation}
\end{prop}
The two terms on the right-hand side have natural interpretations: the first term is the weighted source errors, and the second one measures the distance between the labeling functions from the source domain and target domain.
Compared to \myref{lemma:bound}, \myref{theo:2bound} does not depend on $\lambda_\alpha$, i.e., the choice of the hypothesis class $\mathcal{H}$ makes no difference.
More importantly, the new upper bound in \myref{theo:2bound} reflects the influence of adaptivity gaps between each source domain to the target, \ie $\mathbb{E}_{{\mathcal{D}}_\T}[|f_i-f_\T|]$. The most similar generalization bound to us is~\cite{albuquerque2019generalizing}, in Appendix~\ref{sec:app_tight}, we show that the proposed bound is tighter. Although in this work, the density ratio ${P_{\T}(x)}/{P_{i}(x)}$ is ignored and regarded as a constant, it has an interesting connection between reweighting methods.

\textbf{Connect the density ratio to reweighting methods.}
Intuitively, the density ratio stresses the importance of data sample reweighting, where data samples that are more likely from the target domain should have larger weights. Note that estimating ${P_\T(x)}/{P_i(x)}$ directly is intractable and the term is significantly problematic with no constraint. However, we can make some safe assumptions and obtain applicable formulations, which is exactly what distributionally robust optimization (DRO)~\cite{ben2009robust} does\footnote{The assumption used in DRO such as the distance between the source and target distributions is not so far is safe, because if the distance can be arbitrarily significant, almost all existing theories will be loose and no generalization method can work well.}. Specifically, if we restrict the target domain within a $f$-divergence ball (such as Kullback-Leibler divergence) from the training distribution, which is also known as KL-DRO~\cite{hu2013kullback}, then the density ratio will be converted to a reweighting term ${e^{\ell/\tau^*}}/{\mathbb{E}[e^{\ell/\tau^*}]}$ used for training, where $\ell$ indicates the classification error incurred by $(x,y)$ and $\tau^*$ is a hyperparameter. Namely, the reweighting term is actually an approximate estimation of the density ratio. Existing methods~\cite{liu2021just,zhang2022generalizable,sagawa2020distributionally} use similar reweighting terms and our error bound provides a theoretical explanation for why they work well on DG. Existing methods~\cite{liu2021just,zhang2022generalizable,sagawa2020distributionally} use similar reweighting strategies and our error bound provides a theoretical explanation for why they work well on DG  (See Appendix~\ref{sec:app_density} for formal derivation). 

\section{Domain-specific Risk Minimization}

\begin{figure}[t]
\centering
\includegraphics[width=\linewidth]{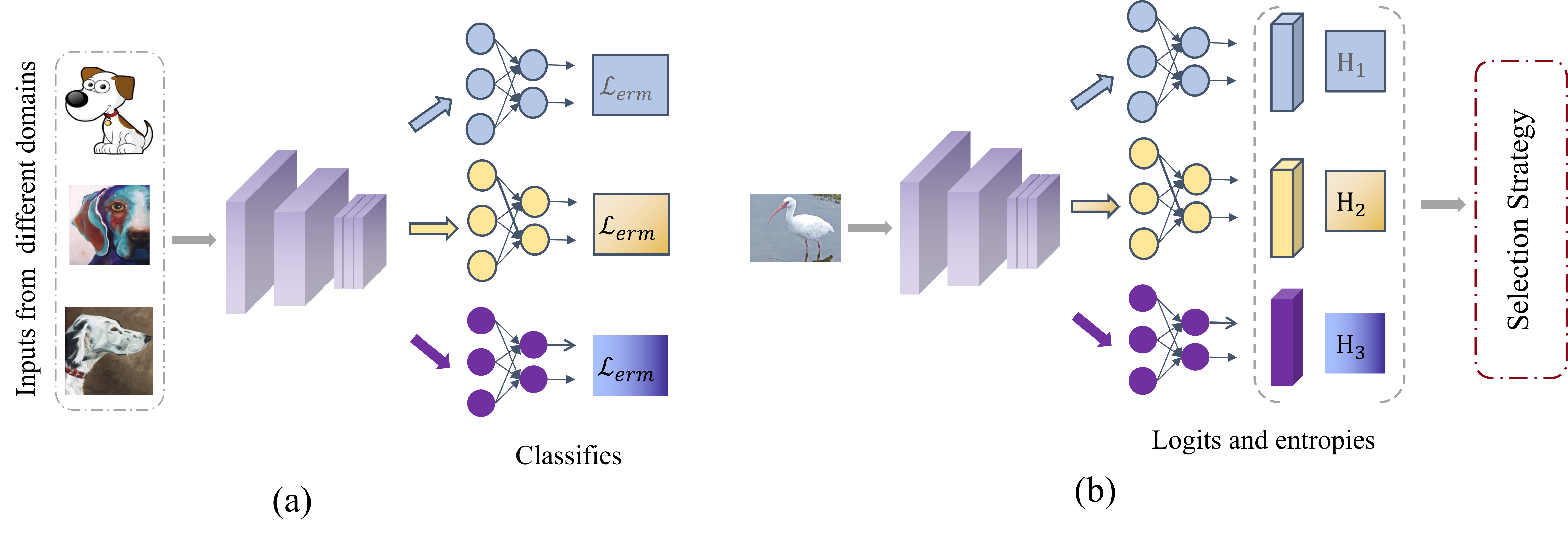}
\caption{\textbf{An illustration of the training and testing pipelines using \abbr.} (a) during training, it jointly optimizes an encoder shared by all domains and the specific classifiers for each individual domain. $\mathcal{L}_{erm}$ indicates the cross-entropy loss function. (b) the new image is first classified by all classifiers and a test-time model selection strategy is applied to generate the final result.}
    \label{fig:model}
\end{figure}

Our error bound in \myref{theo:2bound} suggests a novel perspective on OOD algorithm design. In this paper, we follows the test-time adaptation setting for domain generalization~\cite{iwasawa2021test} and try to utilize the online target samples to minimize the adaptivity gap. However, \myref{theo:2bound} needs to calculate the expectation and the optimal hypothesis function $f_\T$ on the target domain, which are very challenging to obtain. Therefore, we propose a heuristic algorithm, \abbr, which avoids the calculation of intractable terms in \myref{theo:2bound} and approximately minimizes the bound. The main pipeline of the proposed Domain-Specific Risk Minimization (DRM) is shown in~\figurename~\ref{fig:model}.
\subsection{Domain-specific labeling function}
One natural idea is to use \textbf{domain-specific} classifiers $\{\hat{f}_i\}_{i=1}^K$ rather than a shared classifier $\hat{f}$ for source domains.
Each $\hat{f}_i$ is responsible for classification in $\mathcal{D}_i$.
During training, our goal is to minimize $\frac{1}{K}\sum_{i=1}^K\mathbb{E}_{x\sim{\mathcal{D}}_i}\left[|{\color{brown}\hat{f}_i}-f_i|\right]$ by assuming that $K$ training domains are uniformly mixed ($\alpha_i=1/K$).
The generalization results are better with reweighting terms, \eg using GroupDRO~\cite{sagawa2020distributionally}, in the RotatedMNIST dataset, the accuracy of $d=5$ with reweighting terms is $97.3\%$, which is better than $96.8\%$ without reweighting.
We simply ignore the reweighting term in this work since it is not our focus.

Specifically, given $K$ source domains, \abbr~utilizes a \emph{shared} encoder $g$ and a group of prediction head $\{h_i\}_{i=1}^K$ for all domains, respectively.
The encoder is trained by all data samples while each head $h_i$ is trained by images from domain $\mathcal{D}_i$.
It is also possible (but less efficient and accurate) to use specific $g_i$ for each domain.\footnote{Using domain-specific $g_i$ will inevitably increase the computation and memory burden. We observe that $h_i \circ g$ gives an OOD accuracy of $70.1\%$ while the result is only $64.8\%$ for $h_i \circ g_i$ on the Colored MNIST dataset. A possible reason is that a shared encoder $g$ can be seen as an implicit regularization, which prevents the model from overfitting specific domains.}

\subsection{Test-time model selection and adaptation} 

\textbf{Test-time adaptive intuitions from the bound}.
After training, we can get $K$ hypotheses $\hat{f}_i$ that can well approximate source labeling functions.
During testing, our error bound provides two strategies to minimize the second term in the upper bound, \ie $\sum_{i=1}^K\alpha_i\mathbb{E}_{{\mathcal{D}}_\T}[|{f}_i-f_\T|]$, one natural strategy is to find $$\alpha^*=\arg\min \alpha_i\mathbb{E}_{{\mathcal{D}}_\T}[|{f}_i-f_\T|],$$ 
which is termed \textbf{test-time model selection}. The intuition is that if we can find the source domain $\mathcal{D}_{i^*}$ with a labeling function $f_{i^*}$ that minimizes the adaptivity gap $\mathbb{E}_{\mathcal{D}_\T}[|f_{i^*}-f_\T|]$, then we have that $\alpha_{i}=1,\text{iff } i=i^*,\text{otherwise } 0$ will minimize this term.
Second, if we suppose $\hat{f}_i\approx f_i$, then minimizing $\sum_{i=1}^K\mathbb{E}_{{\mathcal{D}}_\T}[|\hat{f}_i-f_\T|]$ will also minimize the bound. The resulting strategy is termed \textbf{test-time retraining}.
Since $f_\T$ is unknown, we can update model parameters by the inferred target pseudo labels or use some unsupervised losses such as entropy minimization.
Note that these two strategies are orthogonal and can be used simultaneously.
In the following, we articulate these two strategies.

\subsubsection{Test-time model selection.} As mentioned above, we can manipulate $\alpha_i$ to affect the second term in our bound: for every test sample $x\in\D_\T$, if we can estimate the adaptivity gap $\{H_i=|f_{i}(x)-f_\T(x)|\}_{i=1}^K$ and choose $i^*=\arg\min_i \{H_i\}_{i=1}^K$. Then $\alpha_{i}=1,\text{iff } i=i^*,\text{otherwise } 0$ makes this term the minimum and the prediction will be $\hat{f}_{i^*}(x)$. The challenge is estimating $\{H_i\}_{i=1}^K$ and we propose two approximations.

\textbf{Similarity Measurement (SM).}
We first reformulate $\alpha_i\mathbb{E}_{{\mathcal{D}}_\T}[|f_i-f_\T|]$ as follows: 
\begin{equation}
\begin{aligned}
&\alpha_i\mathbb{E}_{{\mathcal{D}}_\T}[|f_i-f_\T|]\\
&=\alpha_i\mathbb{E}_{{\mathcal{D}}_\T}\left[|f_i-\mathbb{E}_{\mathcal{D}_i}[f_i]+\mathbb{E}_{\mathcal{D}_i}[f_i]-f_\T|\right]\\
&\leq \alpha_i\left(\mathbb{E}_{{\mathcal{D}}_\T}\left[|f_i-\mathbb{E}_{\mathcal{D}_i}[f_i]|\right]+\mathbb{E}_{{\mathcal{D}}_\T}\left[|\mathbb{E}_{\mathcal{D}_i}[f_i]-f_\T|\right]
\right),
\end{aligned}
\end{equation}
where $f_\T$ is intractable and we then focus on $\mathbb{E}_{{\mathcal{D}}_\T}\left[|f_i-\mathbb{E}_{\mathcal{D}_i}[f_i]|\right]$, which intuitively measures the prediction difference of the given test data $x\in\D_\T$ and the average prediction result in domain $\D_i$.
However, taking the average of the prediction labels might produce ill-posed results\footnote{If all source domains have two data samples with different labels, \eg two different one-hot labels $[0,1],[1,0]$. Then the average prediction result of all source domains will be $[0.5,0.5]$ and have no difference.} and we use $\mathbb{E}_{{\mathcal{D}}_\T}\left[|g-\mathbb{E}_{\mathcal{D}_i}[g]|\right]$ to approximate this term, where we calculate the representation difference between the test sample and the average representation of the domain $\D_i$.
For each $x\in\mathcal{D}_\T$, the estimation $H_i=Dist(g(x), \mathbb{E}_{\mathcal{D}_i}[g])$, \ie the distance between $g(x)$ and the average representation of $\mathcal{D}_i$. 
The $Dist$ function can be any distance metric such as $l_p$-Norm, the negative of cosine similarity, $f-$divergence~\cite{nowozin2016f}, MMD~\cite{li2018domain}, or $\mathcal{A}$-distance~\cite{ben2010theory}.
We use cosine similarity (CSM) and $l_2$-Norm (L2SM) in our experiments for simplicity.

\textbf{Prediction Entropy Measurement (PEM).}
During testing, denote the $K$ individual classification logits as $\{\bar{\textbf{y}}^k\}_{k=1}^K$, where $\bar{\textbf{y}}^k=[y_1^k,...,y_C^k]$, and $C$ is the number of classes.
Given the following assumption: \textit{``the more  confident prediction $h_i\circ g$ makes on $\mathcal{D}_\T$, the more similar $f_i$ and $f_\T$ will be''}.
Then, the prediction entropy of $\bar{\textbf{y}}^k$ can be calculated as $H_k=-\sum_{i=1}^c \frac{y_i^k}{\sum_{j=1}^c y_j^k}\log \frac{y_i^k}{\sum_{j=1}^c y_j^k}$, where the entropy is used as our expected estimation.
In our experiments, we find that the prediction entropy is consistent with domain similarities, which is similar to SM. 

\paragraph{Model Ensembling.}
A one-hot mixed weight is too deterministic and cannot fully utilize all learned classifiers.
\textbf{Softing mixed weights}, on the other hand, can further boost generalization performance and enlarge the hypothesis space, \ie for ERM, we can generate the final prediction as $\sum_{k=1}^K \bar{\textbf{y}}_k {H_k^{-\gamma}}/{\sum_{i=1}^K H_i^{-\gamma}}$, where $H_k^{-\gamma}$ indicates the contribution of each classifier.
We use $-\gamma$, but not $\gamma$ since the smaller the adaptivity gap, the larger the contribution of $f_i$ should be. Specifically, for $\gamma=0$, we then have a uniform combination, \ie $\alpha_i=1/K,\forall i\in[1,2,...,K]$; for $\gamma\rightarrow\infty$, we then have a one-hot weight vector with $\alpha_{i}=1\text{ iff }i=i^*\text{ otherwise } 0$. In experiments, we compare the different selection strategies and PEM generally performs the best, thus \textbf{we use PEM by default}.

\subsubsection{Test-time retraining.}

\begin{wraptable}{r}{4.5cm}
\centering
\small
\vspace{-0.3cm}
\begin{tabular}{ccc}
\toprule
Method & \textbf{Clf}  & \textbf{Full} \\ \midrule
ERM        & 80.3          & 80.3          \\
DRM                    & 83.0            & 83.0            \\\hline
Vanilla retraining      & 83.0            & 83.8          \\
DRM retraining              & \textbf{84.1} & \textbf{84.8} \\
\bottomrule
\end{tabular}%
\caption{Different pseudo label generalization methods.}
\label{tab:pseudo_label}
\vspace{-0.3cm}
\end{wraptable} 
The simplest idea to retrain the model is that, for each prediction head, we use the argmax of the prediction result as pseudo labels and then train the model by cross-entropy loss, which is termed \textit{Vanilla Retraining}.
However, it performs poorly (Table~\ref{tab:pseudo_label}) no matter only tuning the prediction heads (Clf) or the overall model (Full).
Thanks to the domain-specific classifiers, we can produce more reliable pseudo labels.
Specifically, we generate pseudo labels by the weighted mix of predictions by all prediction heads where the weights are just mixed weights in the model selection phase.
We compare these generation strategies on the PACS dataset with `A' as the target.
Table~\ref{tab:pseudo_label} shows that with the proposed pseudo-label generation strategy, the retraining process can be better guided.


\textbf{Remark.} Although our algorithm is mostly heuristic, we show experimentally that by modeling domain-specific labeling functions, \abbr can further reduce source errors (\ie the first term in our upper bound); For the second term, the test-time model selection and retraining reduce the adaptivity gap by enriching hypothesis class and target sample retraining, leading to superior generalization capability. In the following analysis, we show that the proposed \abbr performs well on the counterexample in Section~\ref{sec:theory}.

\abbr~can attain near $0$ source error in the above-mentioned counterexample by using $g(X)=X$ and  
\begin{equation}
h_r(X)=\left\{         
  \begin{array}{cc} 
    0 & ~~x_1\leq -3\\ 
    1 & ~~x_1>-3\\ 
  \end{array}\right.,~~~~h_b(X)=\left\{  
  \begin{array}{cc} 
    1 & ~~x_1\leq 3\\ 
    0 & ~~x_1>3\\ 
  \end{array}\right..\nonumber
\end{equation}
Furthermore, the choice of $g$ is not a matter, and we can easily generalize it to other cases. For example, given $g(X)=\mathbb{I}_{x_1<0}\cdot(x_1+3)+\mathbb{I}_{x_1>0}\cdot(x_1-3)$ for an invariant representation. \abbr~can still attain $0$ source error by using 
\begin{equation}
h_r(X)=\left\{   
  \begin{array}{cc} 
    0 & ~~x_1\leq 0\\ 
    1 & ~~x_1>0\\ 
  \end{array}\right.,~~~~h_b(X)=\left\{  
  \begin{array}{cc} 
    1 & ~~x_1\leq 0\\ 
    0 & ~~x_1>0\\ 
  \end{array}\right.\nonumber
\end{equation}
Taking into account the PEM test-time model selection strategy, \eg in the counterexample, $\D_o$ is more similar to $\D_r$ than to $\D_b$, hence the entropy when $X\in\D_o$ is classified by $h_r$ is less than the entropy classified by $h_b$. In this way,~\figurename~\ref{fig:examples_drm} shows that the learned classification boundaries can achieve test errors near $0$ in both the unseen target domains $\D_o$ and $\D_g$.

\section{Experimental Results}\label{sec:exp}

We first conduct case studies on a popular \textbf{correlation shift} dataset (Colored MNIST).
Then, we compare \abbr with other advanced methods on DG benchmarks (\textbf{diversity shift}). The results verify the argument in the introduction: by utilizing the target data during inference, we can better robustify a model to both distribution shifts.
We also compare \abbr with different test-time adaptive methods with various backbones.
For fair comparisons, We use test-time retraining just when compared to test-time adaptation methods, namely \textbf{\abbr denotes the method wo/ retraining}.

\textbf{Experimental Setup.} We use five popular OOD generalization benchmark datasets: Colored MNIST~\cite{arjovsky2020invariant}, Rotated MNIST~\cite{ghifary2015domain}, PACS~\cite{li2017deeper}, VLCS~\cite{torralba2011unbiased}, and DomainNet~\cite{peng2019moment}. We compare our model with ERM \cite{vapnik1998statistical}, IRM \cite{arjovsky2020invariant}, Mixup \cite{yan2020improve}, MLDG \cite{li2018learning}, CORAL \cite{sun2016deep}, DANN \cite{ganin2016domain}, CDANN \cite{li2018deep}, MTL~\cite{blanchard2021domain}, SagNet~\cite{nam2021reducing}, ARM~\cite{zhang2021adaptive}, VREx~\cite{krueger2021out}, RSC~\cite{huang2020self}, Fish~\cite{shi2022gradient}, and Fishr~\cite{rame2021fishr}. All the baselines in DG tasks are implemented using the codebase of Domainbed \cite{gulrajani2021in}.

\textbf{Hyperparameter search.} Following the experimental settings in \cite{gulrajani2021in}, we conduct a random search of 20 trials over the hyperparameter distribution for each algorithm and test domain. Specifically, we split the data from each domain into $80\%$ and $20\%$ proportions, where the larger split is used for training and evaluation, and the smaller ones are used for select hyperparameters. We repeat the entire experiment twice using different seeds to reduce randomness. Finally, we report the mean over these repetitions as well as their estimated standard error. We observe that the proposed \abbr~does not converge within $5k$ iterations on the DomainNet dataset and we thus train it with an extra $5k$ iterations. 

\textbf{Implementation details.} During training, we use the average of all classifiers' losses as the training loss. To further enlarge the hypothesis space, we can simply add an additional prediction head that is trained by all data samples, namely, we have a total of $K+1$ prediction heads in the test phase, such a simple trick is optional and can bring performance gains on some of our benchmarks.

\textbf{Model selection} in domain generalization is intrinsically a learning problem, and we use test-domain validation, one of the three methods in \cite{gulrajani2021in}. This strategy is an oracle-selection one since we choose the model maximizing the accuracy on a validation set that follows the distribution of the test domain.

\textbf{Model architectures.}
Following \cite{gulrajani2021in}, we use as encoders ConvNet for RotatedMNIST (detailed in Appdendix D.1 in \cite{gulrajani2021in}) and ResNet-50 for the remaining datasets.

See Appendix~\ref{sec:data_detail} for dataset details.

\begin{table*}[t]
\centering
\adjustbox{max width=\textwidth}{%
\setlength{\tabcolsep}{4.0pt}
\begin{tabular}{@{}ccccccccc@{}}
\toprule                                                                & \multicolumn{2}{c}{\textbf{+90\%} ($d=0$) }                                                            & \multicolumn{2}{c}{\textbf{+80\%} ($d=1$)}                                                            & \multicolumn{2}{c}{\textbf{-90\%} ($d=2$)}                                                           & \multicolumn{2}{c}{Avg}                                                           \\ 
Method                                                      &  train &  test &  train &  test &  train &  test &  train &  test \\\midrule
ERM                                                            & \textbf{86.1$\pm$3.9}                           & 71.8$\pm$0.4                         & 83.6$\pm$0.5                           & 72.9$\pm$0.1                         & 87.5$\pm$3.4                           & 28.7$\pm$0.5                         & 85.7                                     & 57.8                                   \\
IRM                                                            & 78.2$\pm$9.5                           & 72.0$\pm$0.1                         & 70.6$\pm$9.1                           &  72.5$\pm$0.3 & 85.3$\pm$4.7                           & \textbf{58.5$\pm$3.3}                         & 78.0                                       & 67.7                                   \\\rowcolor{Gray}
\textbf{DRM}                                                           & 81.8$\pm$9.8                           & \textbf{86.7$\pm$2.4}                         & 90.2$\pm$0.2                           & 80.6$\pm$0.2                         & 88.0$\pm$4.5                           & 43.1$\pm$7.5                         & 86.7                                     & 70.1                                   \\\rowcolor{Gray}
\textbf{+CORAL}                                                        & 83.4$\pm$8.6                           & 85.3$\pm$2.3                         & 9\textbf{1.6$\pm$0.7}                           & \textbf{80.7$\pm$0.2}                         & \textbf{89.4$\pm$4.9 }                          & 47.2$\pm$3.6                         & \textbf{88.1}                                     & \textbf{71.1}                                   \\\hline
 RG         & 50                                       & 50                                     & 50                                       & 50                                     & 50                                       & 50                                     & 50                                       & 50                                     \\
 OIM & 75                                       & 75                                     & 75                                       & 75                                     & 75                                       & 75                                     & 75                                       & 75                                     \\
 ERM (gray)         &  84.8$\pm$2.7   &  73.9$\pm$0.3 & 84.3$\pm$1.4                           & 73.7$\pm$0.4                         & 83.4$\pm$2.3                           & 73.8$\pm$0.7                         & 84.2                                     & 73.8                                   \\ \bottomrule
\end{tabular}}
\caption{\textbf{Accuracies ($\%$) of different methods for the Colored MNIST synthetic task.} OIM (optimal invariant model) and RG (random guess) are hypothetical mechanisms.}
\label{tab:case_mnist}
\end{table*}

\begin{table*}[t]
\centering
\adjustbox{max width=\textwidth}{%
\begin{tabular}{lcccccc}
\toprule
Method        & \textbf{CMNIST}     & \textbf{RMNIST}     & \textbf{VLCS}             & \textbf{PACS}              & \textbf{DomainNet}        & \textbf{Avg}              \\
\midrule
ERM~\cite{vapnik1998statistical}                       & 57.8 $\pm$ 0.2            & 97.8 $\pm$ 0.1            & 77.6 $\pm$ 0.3            & 86.7 $\pm$ 0.3            & 41.3 $\pm$ 0.1            & 72.2                      \\
IRM~\cite{arjovsky2020invariant}                       & 67.7 $\pm$ 1.2            & 97.5 $\pm$ 0.2            & 76.9 $\pm$ 0.6            & 84.5 $\pm$ 1.1            & 28.0 $\pm$ 5.1            & 70.9                      \\
GDRO~\cite{sagawa2020distributionally}                  & 61.1 $\pm$ 0.9            & 97.9 $\pm$ 0.1            & 77.4 $\pm$ 0.5            & 87.1 $\pm$ 0.1            & 33.4 $\pm$ 0.3            & 71.4                      \\
Mixup~\cite{yan2020improve}                     & 58.4 $\pm$ 0.2            & 98.0 $\pm$ 0.1            & 78.1 $\pm$ 0.3            & 86.8 $\pm$ 0.3            & 39.6 $\pm$ 0.1            & 72.2                      \\
CORAL~\cite{sun2016deep}                      & 58.6 $\pm$ 0.5            & 98.0 $\pm$ 0.0            & 77.7 $\pm$ 0.2            & 87.1 $\pm$ 0.5            &  41.8 $\pm$ 0.1            & 72.6                      \\
DANN~\cite{ganin2016domain}                      & 57.0 $\pm$ 1.0            & 97.9 $\pm$ 0.1            & 79.7 $\pm$ 0.5            & 85.2 $\pm$ 0.2            & 38.3 $\pm$ 0.1            & 71.6                     \\
CDANN~\cite{li2018deep}                     & 59.5 $\pm$ 2.0            & 97.9 $\pm$ 0.0            & 79.9 $\pm$ 0.2            & 85.8 $\pm$ 0.8            & 38.5 $\pm$ 0.2            & 72.3                      \\
MTL~\cite{blanchard2021domain}                       & 57.6 $\pm$ 0.3            & 97.9 $\pm$ 0.1            & 77.7 $\pm$ 0.5            & 86.7 $\pm$ 0.2            &  40.8 $\pm$ 0.1            & 72.1                      \\
SagNet~\cite{nam2021reducing}                    & 58.2 $\pm$ 0.3            & 97.9 $\pm$ 0.0            & 77.6 $\pm$ 0.1            & 86.4 $\pm$ 0.4            &  40.8 $\pm$ 0.2            & 72.2                     \\
ARM~\cite{zhang2021adaptive}                       & 63.2 $\pm$ 0.7            & 98.1 $\pm$ 0.1            & 77.8 $\pm$ 0.3            & 85.8 $\pm$ 0.2            &  36.0 $\pm$ 0.2            & 72.2                      \\
VREx~\cite{krueger2021out}                      & 67.0 $\pm$ 1.3            & 97.9 $\pm$ 0.1            & 78.1 $\pm$ 0.2            & 87.2 $\pm$ 0.6            & 30.1 $\pm$ 3.7            & 72.1                     \\
RSC~\cite{huang2020self}                        & 58.5 $\pm$ 0.5            & 97.6 $\pm$ 0.1            & 77.8 $\pm$ 0.6            & 86.2 $\pm$ 0.5            &  38.9 $\pm$ 0.6            & 71.8                     \\
Fish~\cite{shi2022gradient}                       & 61.8 $\pm$ 0.8            & 97.9 $\pm$ 0.1            & 77.8 $\pm$ 0.6            & 85.8 $\pm$ 0.6            &  \textbf{43.4 $\pm$ 0.3}            & 73.3                     \\
Fishr~\cite{rame2021fishr}                       & 68.8 $\pm$ 1.4            & 97.8 $\pm$ 0.1            & 78.2 $\pm$ 0.2            & 86.9 $\pm$ 0.2            &  41.8 $\pm$ 0.2            & 74.7                     \\\rowcolor{Gray}
\textbf{\abbr}~                 &  70.1 $\pm$ 2.0       & 98.1 $\pm$ 0.2     & \textbf{80.5 $\pm$ 0.3} & \textbf{88.5 $\pm$ 1.2}   & 42.4 $\pm$ 0.1  &   75.9        \\
\rowcolor{Gray}
\textbf{DRM+CORAL}                 & \textbf{71.1 $\pm$ 1.3}       & \textbf{98.3 $\pm$ 0.1}    &  79.5 $\pm$ 2.4 & 88.4 $\pm$ 0.9    &   { 42.7} $\pm$ 0.1    &   \textbf{76.0}    \\
\bottomrule
\end{tabular}
}
\caption{Out-of-distribution generalization performance. No retraining is applied for a fair comparison.}
\label{tab:ood}
\end{table*}

\begin{figure*}[t]
    \subfigure[]{
    \includegraphics[width=0.23\textwidth]{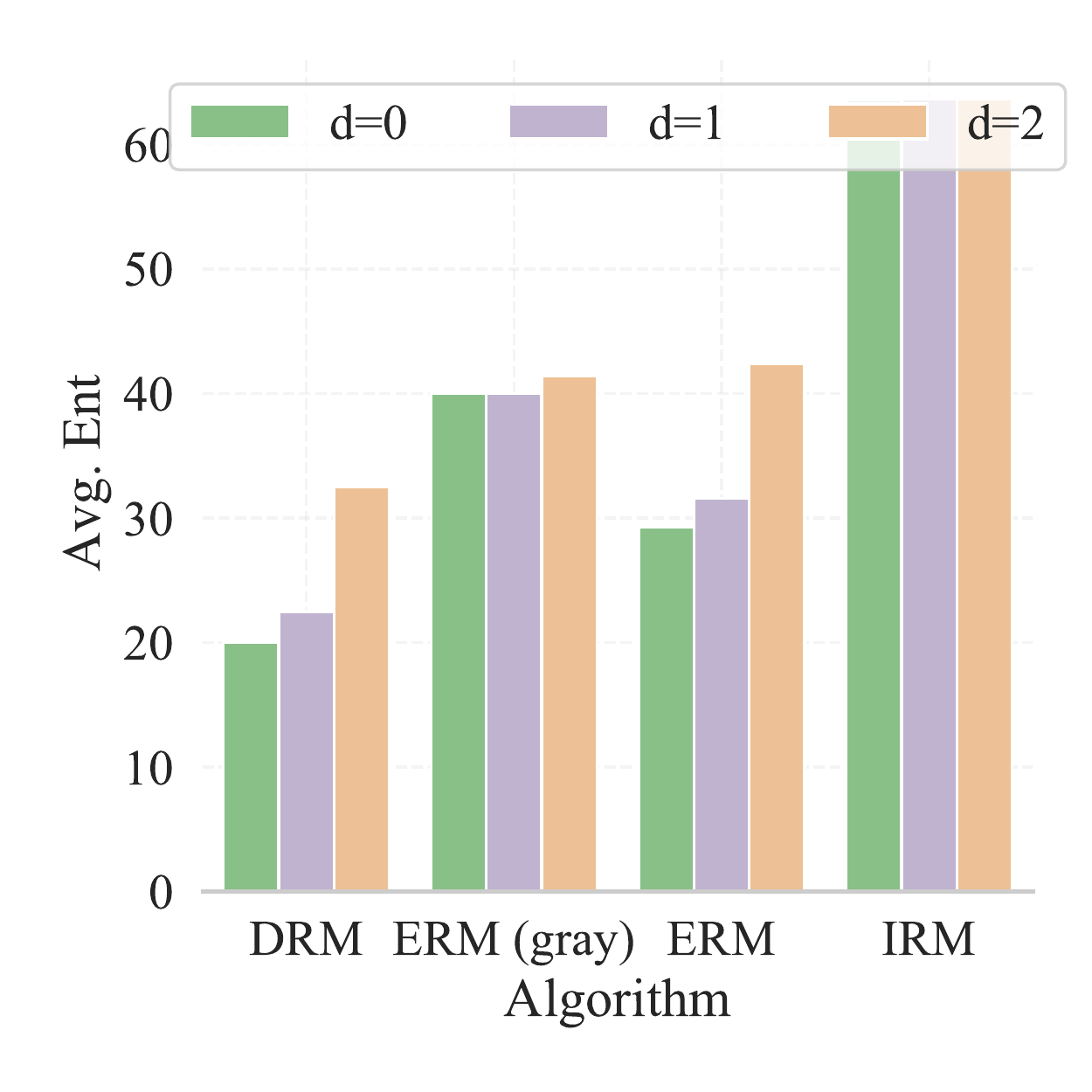}
\label{fig:entropy_cmnista}
    }
    \subfigure[]{
    \includegraphics[width=0.23\textwidth]{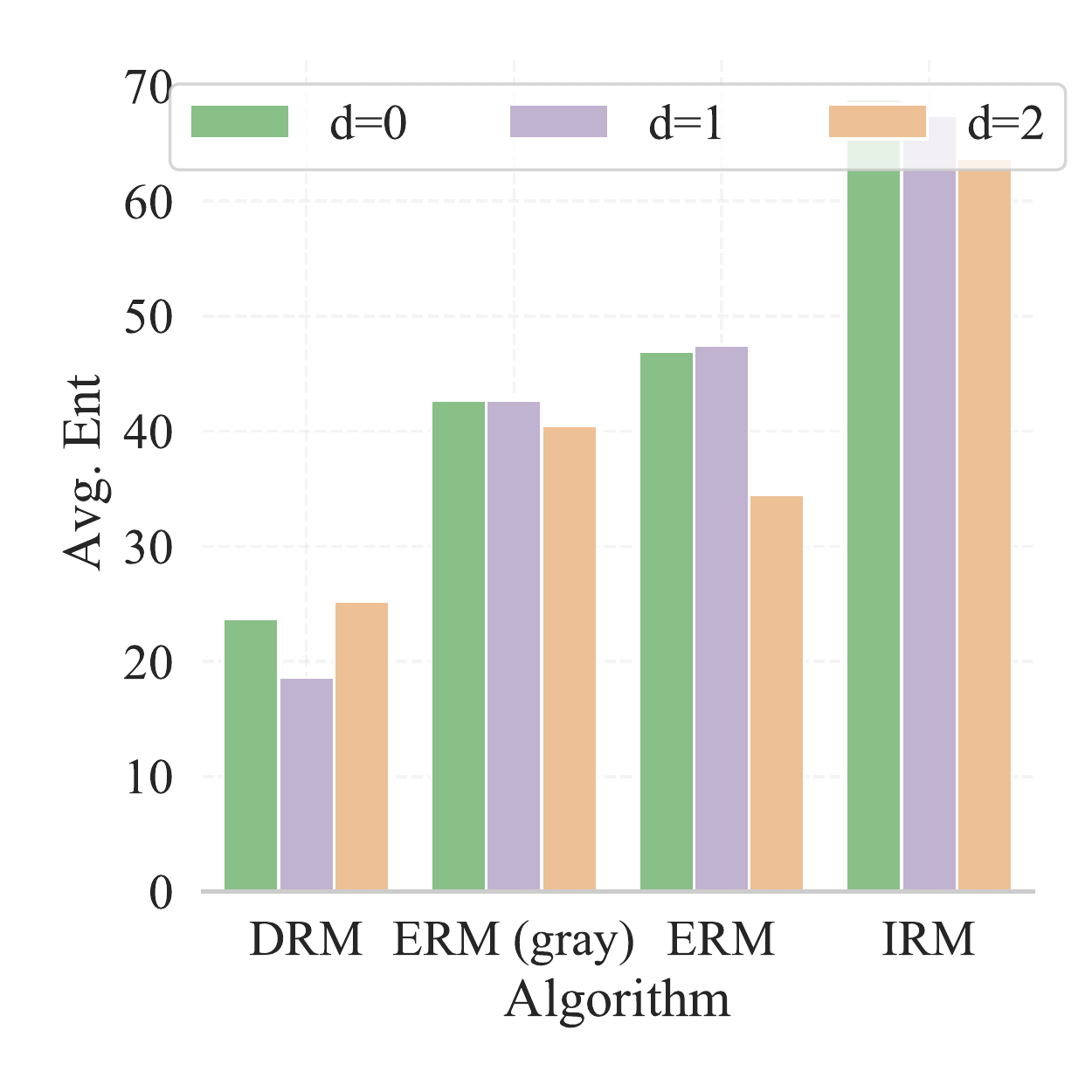}
\label{fig:entropy_cmnistb}
    }
     \subfigure[]{
    \includegraphics[width=0.23\textwidth]{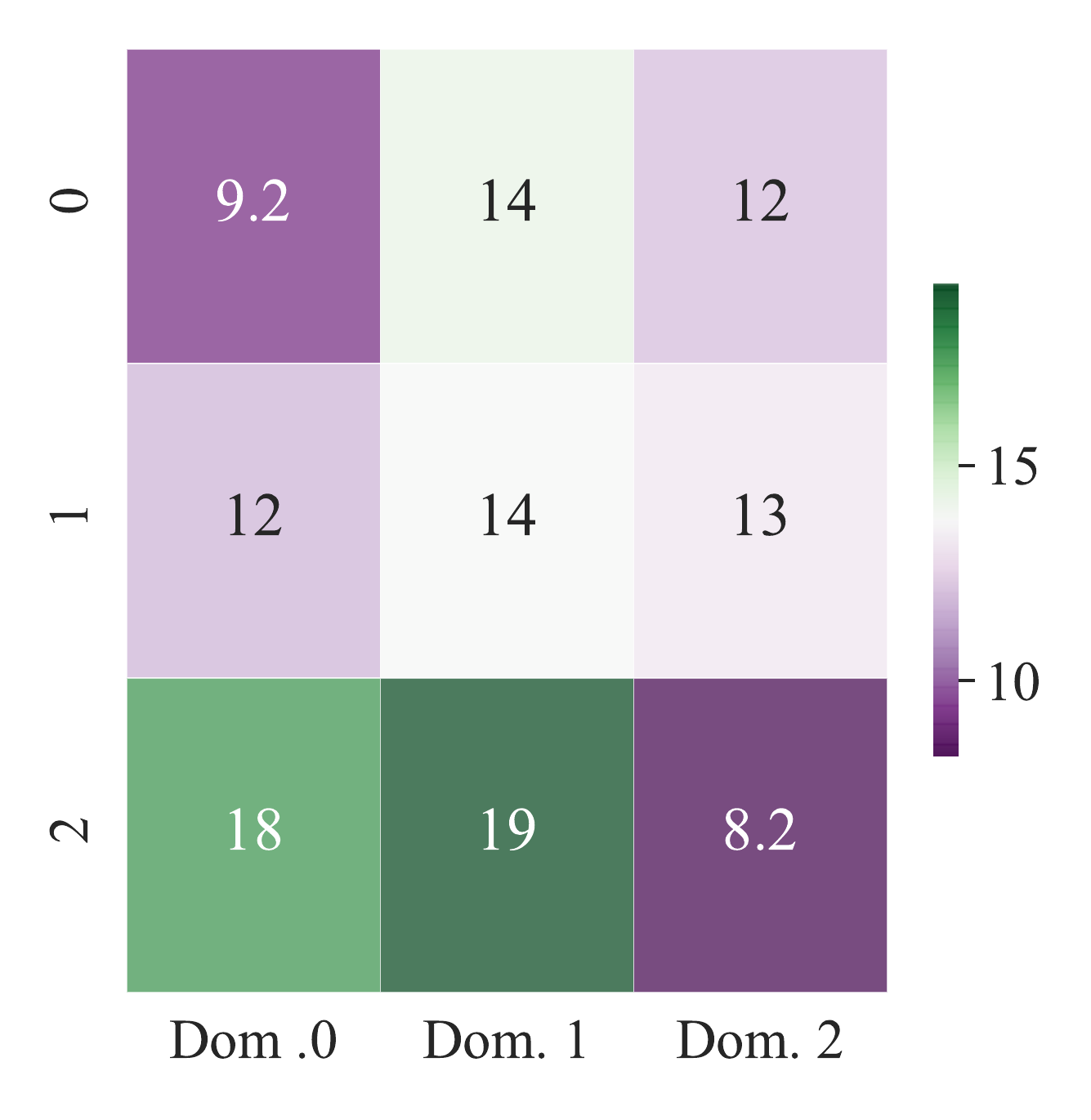}
\label{fig:entropy_cmnistc}
    }
     \subfigure[]{
    \includegraphics[width=0.23\textwidth]{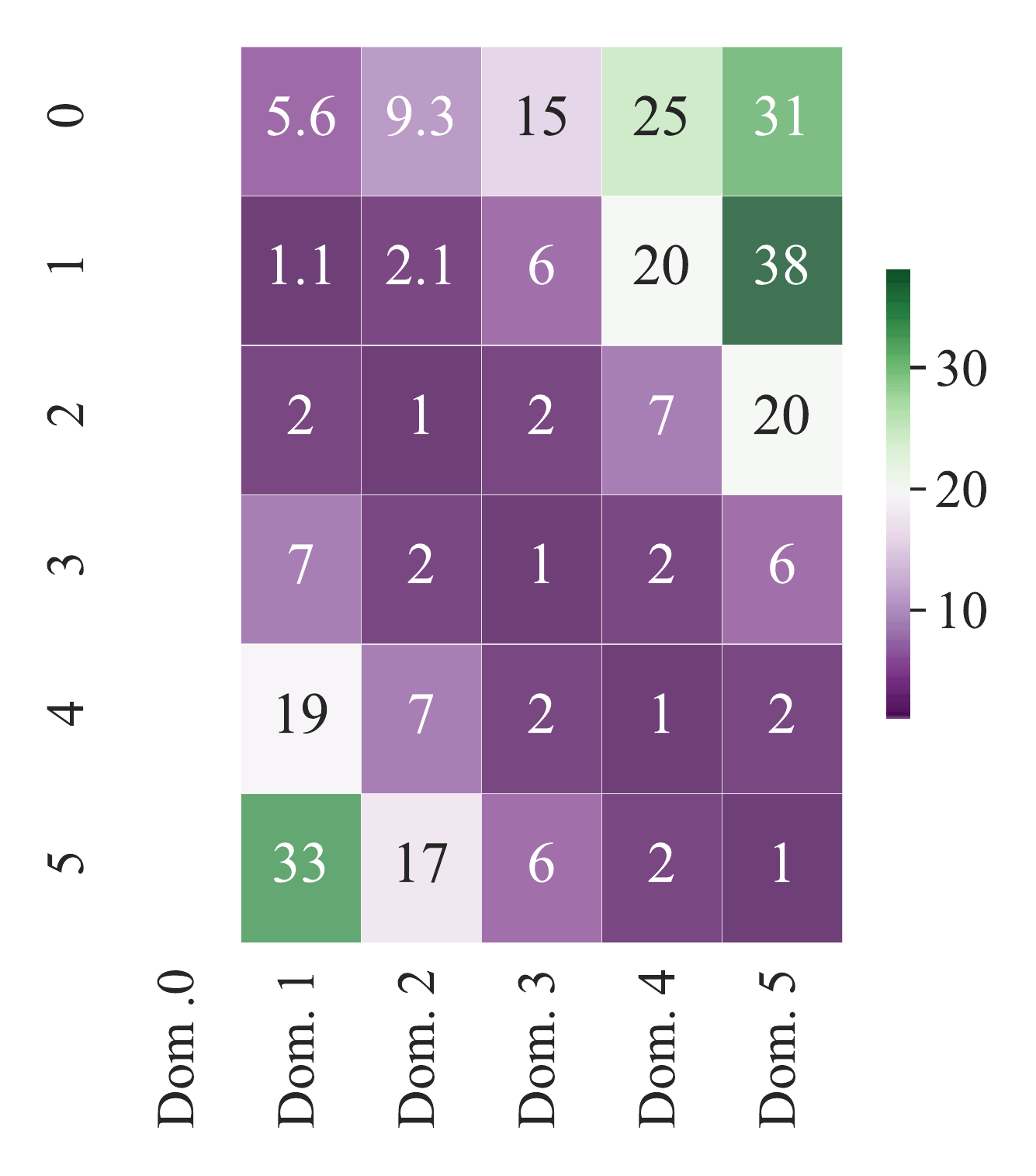}
\label{fig:heatmapsa}
    }
    \caption{\textbf{The entropy of different predictions.} (a) Training domain $\{0,1\}$ and testing domain $\{2\}$. (b) The average of training/testing domains $\{0,1\}$/$\{2\}$, $\{0,2\}$/$\{1\}$, and $\{1,2\}$/$\{0\}$. (c) Domain-classifier correlation matrix, the value $v_{ij}$ is the entropy of predictions incurred by predicting samples in the domain $i$ with classifier $j$. Dom.$i$ indicates the classifier for the domain $d=i$. (d) Domain-classifier correlation matrices on Rotated MNIST.} 
    \label{fig:entropy_cmnist}
\end{figure*}

\subsection{Case studies on correlation shift datasets}
\label{sec:exp_cminst}

In the following, we conduct thorough experiments and analyze a popular correlation shift benchmark, \ie the ColoredMNIST dataset~\cite{arjovsky2020invariant}.
It constructs a binary classification problem based on the MNIST dataset (digits 0-4 are class one and 5-9 are class two). Digits in the dataset are either colored red or green, and there is a strong correlation between color and label but the correlations vary across domains. For example, green digits have a $90\%$ chance of belonging to class 1 in the first domain $+90\%(d=0)$, and a $10\%$ chance of belonging to class 1 in the third domain $-90\%(d=2)$.

\textbf{\abbr has superior generalization ability on the dataset with correlation shift.}
As shown in Table~\ref{tab:case_mnist}, ERM achieves high accuracy in training domains, but lower chance accuracy in the test domain due to its reliance on spurious correlations.
IRM~\cite{arjovsky2020invariant} forms a trade-off between training and testing accuracy. An ERM model trained on only gray images, \ie ERM (gray), is perfectly invariant by construction and attains a better tradeoff than IRM. The upper bound performance of invariant representations (OIM) is a hypothetical model that not only knows all spurious correlations but also has no modeling capability limit.
For averaged generalization performance, \abbr, without any invariance regularization, outperforms IRM by a large margin ($>2.4\%$).
In addition, \textit{the source accuracy of \abbr is even higher than ERM and significantly higher than IRM and OIM}.
Note that \abbr is complementary to invariant learning-based methods, where the incorporation of CORAL can further boost both training and testing performances.
Though the Colored MNIST dataset is a good indicator to show the model capacity for avoiding spurious correlation, these spurious correlations are unrealistic and utopian.
Therefore, when testing on large DG benchmarks, ERM outperforms IRM.
Unlike them, \abbr not only performs well on the semi-synthetic dataset but also attains state-of-the-art performance on large benchmarks.

\textbf{PEM implicitly reduces prediction entropy and the entropy-based strategy performs well in finding a proper labeling function for inference.}
The prediction entropy is often related to the fact that more confident predictions tend to be correct~\cite{wang2020tent}. 
\figurename~\ref{fig:entropy_cmnista} shows that the entropy in target domain ($d=2$) tends to be greater than the entropy in source domains, where the source domain with stronger spurious correlations ($d=1$) also has larger entropy than easier one ($d=0$).
Fortunately, with the entropy minimization strategy, we can find the most confident classifier for a given data sample, and \abbr can reduce the prediction entropy (\figurename~\ref{fig:entropy_cmnistb}).
To further analyze the entropy minimization strategy, we visualize the domain-classifier correlation matrix in \figurename~\ref{fig:entropy_cmnistc}, where the entropy between the domain and its classifier is the minimal, verifying the efficacy of the PEM strategy.

\subsection{Results on general OOD benchmarks}

\textbf{OOD results.}
The average OOD results on all benchmarks are shown in Table~\ref{tab:ood}.
We observe consistent improvements achieved by \abbr compared to existing algorithms. The results indicate the superiority of \abbr in real-world diversity shift datasets. See the Appendix for multi-target domain generalization and detailed performance on every domain.

\textbf{In-distribution results.} Current DG methods ignore the performance of source domains since they focus on target results. However, source domain performance is also of great importance in applications~\cite{yang2021generalized}, i.e., the in-distribution performance.
We then show the in-distribution performances of VLCS and PACS in Table~\ref{tab:idp-vlcs-pacs}.
\abbr achieves comparable or superior performance in the source domains compared to ERM and beats IRM by a large margin, which indicates that \abbr achieves satisfying in- and out-distribution performance.

\begin{table}[]
    \centering

\adjustbox{max width=\columnwidth}{%
    \begin{tabular}{lccccc}
    \toprule
    \multirow{2}{*}{Method}  & 
\multicolumn{5}{c}{\textbf{VLCS}} \\

  & C & L & S & V & Avg \\\midrule
  ERM & 78.2$\pm$3.3  & 87.8$\pm$9.0  & 86.3$\pm$10.2  & 83.3$\pm$11.6  & 83.9\\
  IRM         & 76.9$\pm$2.9  & \textbf{88.2$\pm$8.9}  & 85.3$\pm$9.8  & 77.3$\pm$1.0  &   81.9 \\
  \textbf{DRM}  &  \textbf{78.5$\pm$2.9} & 87.2$\pm$9.2  &  \textbf{87.3$\pm$9.0} &  \textbf{84.0$\pm$10.9} & \textbf{84.3}\\\midrule
 \multirow{2}{*}{Method}& \multicolumn{5}{c}{\textbf{PACS}} \\ 
 & A & C & P & S & Avg \\\midrule
 ERM & 96.7$\pm$0.3  & 96.4$\pm$1.5  & \textbf{95.3$\pm$1.2}  &  \textbf{96.3$\pm$0.1}& 96.2 \\
 IRM  & 95.9$\pm$1.6  & 94.2$\pm$2.5 &  94.3$\pm$1.0 & 94.5$\pm$1.8 & 94.7   \\
 \textbf{DRM}  & \textbf{96.9$\pm$0.3} &  \textbf{96.4$\pm$1.3} &  95.2$\pm$0.9 & 96.1$\pm$0.6 & \textbf{96.2} \\ \bottomrule
    \end{tabular}
}
\caption{In-distribution performance on VLCS and PACS.}
\label{tab:idp-vlcs-pacs}
\end{table}
\begin{table}[t]
\resizebox{\linewidth}{!}{%
\begin{tabular}{@{}lcc|l|lcc@{}}
\toprule
Method & BSZ=32 & BSZ &  & Method & BSZ=32 & BSZ=8 \\ \midrule
ResNet50 & 83.98 & 83.98 &  & {\color{brown}ResNet50} & 83.98 & 83.98 \\
PLClf & 85.63 & 85.55 &  & DRM & 86.57 & 86.57 \\
PLFull & 86.50 & 85.88 &  & +Retrain Cls & 87.90 & 87.83 \\
SHOT & 86.53 & 85.85 &  & +Retrain Full & 89.30 & 89.33 \\\cline{5-7}
SHOTIM & 86.40 & 85.68 &  & {\color{brown}ResNet18} & 79.98 & 79.98 \\
T3A & 86.23 & 86.00 &  & DRM & 80.30 & 80.30 \\
ResNet50-BN & 83.18 & 83.18 &  & +Retrain Cls & 82.95 & 82.18 \\
TentClf & 84.15 & 84.15 &  & +Retrain Full & 84.70 & 84.35 \\\cline{5-7}
TentNorm & 85.60 & 84.00 &  & {\color{brown}ViT-B16} & 87.10 & 87.10 \\\cline{0-2} 
DRM & 86.57 & 86.57 &  & DRM & 87.85 & 87.85 \\
+Retrain Cls & 87.90 & 87.83 &  & +Retrain Cls & 90.08 & 90.08 \\
+Retrain Full & 89.30 & 89.33 &  & +Retrain Full & 90.95 & 90.85 \\ \bottomrule
\end{tabular}%
}
\caption{ (Left) Comparison of our method and existing test-time adaptation methods on PACS. (Right){Domain generalization accuracy with different backbone networks on PACS.} The reported number is the average generalization performance over P, A, C, S four domains.}\label{tab:backbons}\label{tab:adapt_comp}
\end{table}
\begin{table*}[th]
\centering
\begin{tabular}{@{}cccccc@{}}
\toprule                                                      
Method          & A                       & C                       & P                       & S                       & Avg           \\ \midrule
DRM w/ Uniform weight & 81.2 $\pm$ 2.2          & 71.2 $\pm$ 1.2          & 93.7 $\pm$ 0.3          & 78.6 $\pm$ 1.5          & 81.2          \\
DRM w/ CSM             & 83.0 $\pm$ 2.1          & 74.6 $\pm$ 2.5          & 95.6 $\pm$ 0.8          & 80.4 $\pm$ 1.2           & 83.4          \\
DRM w/ NNM          & 85.5 $\pm$ 2.4          & 76.8 $\pm$ 2.0          & 96.6 $\pm$ 0.4          & 81.8 $\pm$ 1.5          & 85.2          \\
DRM w/ L2SM               & 87.7 $\pm$ 1.7          & 80.0 $\pm$ 0.5          & 96.0 $\pm$ 1.6          & 82.1 $\pm$ 1.2          & 86.5          \\
\textbf{DRM w/ PEM}    & \textbf{88.3 $\pm$ 2.9} & \textbf{80.1 $\pm$ 0.8} & \textbf{97.0 $\pm$ 0.5} & \textbf{80.9 $\pm$ 0.7} & \textbf{86.6} \\ \bottomrule
\end{tabular}
\caption{Comparison of different test-time model selection strategies on the PACS dataset.}
\label{tab:test-time}
\vspace{-0.2cm}
\end{table*}
\begin{table*}[t]
\centering
\small
\setlength{\tabcolsep}{7.25pt}
\begin{tabular}{@{}crrrrrr@{}}
\toprule
\multirow{2}{*}{Method} &
  \multicolumn{2}{c}{Colored MNIST} &
  \multicolumn{2}{c}{Rotated MNIST} &
  \multicolumn{2}{c}{PACS} \\ \cmidrule(l){2-7} 
\multicolumn{1}{l}{\multirow{-2}{*}{}} &
  Time (sec) &
  \# Params (M) &
  Time (sec) &
  \# Params (M) &
  \multicolumn{1}{c}{Time (sec)} &
  \# Params (M) \\ \cmidrule(r){1-7}
ERM   & 71.02  & 0.3542 & 168.32 & 0.3546 & 2,717.5 & 22.4326             \\
IRM   & 101.49 & 0.3542 & 236.80  & 0.3546 &                      2,786.3      & 22.4326             \\
ARM   & 161.51 & 0.4573 & 360.69 & 0.4562 &                       6,616.9     & 22.5398 \\
FISH  & 137.17 & 0.3542 & 251.76 & 0.3546 &                      23,849.5      & 22.4326             \\\rowcolor{Gray}
\textbf{DRM}   & 83.39  & 0.3544 & 203.15 & 0.3595 &                          2,895.1  & 22.46 \\ \bottomrule
\end{tabular}
\caption{Comparisons of different methods on the number of parameters and training time.}\label{tab:params_all}
\vspace{-0.3cm}
\end{table*}



\textbf{Comparison with test-time adaptative methods.}
For fair comparisons, following~\cite{iwasawa2021test}, the base models (ERM and \abbr) are trained only on the default hyperparameters and without the fine-grained parametric search. Because~\cite{gulrajani2021in} omits the BN layer from ResNet when fine-tuning on source domains, we cannot simply use BN-based methods on the ERM baseline.
For these methods, their baselines are additionally trained on ResNet-50 with BN.
Models with the highest IID accuracy are selected and all test-time adaptation methods are applied to improve generalization performance.
The baselines include Tent~\cite{wang2020tent}, T3A~\cite{iwasawa2021test}, pseudo labeling (PL)~\cite{lee2013pseudo}, SHOT~\cite{liang2020we}, and SHOT-IM~\cite{liang2020we}.
For methods that use gradient backpropagation, we implement both the update of the prediction head (Clf) and the full model (Full).
Results in Table~\ref{tab:adapt_comp} show that: (i) Simply retraining the classifier or the full model by its own prediction is comparable to existing methods; (ii) Tent~\cite{wang2020tent} is sensitive to batch size but the proposed \abbr is not; (iii) The performance of DRM without retraining attains comparable results compared to existing methods, and when incorporated by the proposed retraining method, the performance beats all baselines by a large margin.

\textbf{Results of various backbones.} We conduct experiments with various backbones in Table~\ref{tab:backbons}, including ResNet-50, ResNet-18, and Vision Transformers (ViT-B16).
\abbr achieves consistent performance improvements compared to ERM.
Specifically, \abbr improves $5.3\%, 4.7\%$, and $3.9\%$ for ResNet-50, ResNet-18, and ViT-B16 with evaluation batch size (BSZ) $32$, respectively.

\textbf{Multi-target domain generalization.} IRM~\cite{arjovsky2020invariant} introduces specific conditions for an upper bound on the number of training environments required such that an invariant optimal model can be obtained, which stresses the importance of several training environments. In this paper, we reduce the training environments on the Rotated MNIST from five to three. As shown in Table~\ref{tab:multi-target}, as the number of training environments decreases, the performance of IRM decreases significantly (\eg the average accuracy from $97.5\%$ to $91.8\%$), and the performance on the most challenging domains $d=\{0,5\}$ declines the most ($94.9\%\rightarrow80.9\%$ and $95.2\%\rightarrow91.1\%$). In contrast, both ERM and \abbr~retain high generalization performances while \abbr~outperforms ERM on domains $d=\{0,5\}$. 

\begin{table*}[t]
\centering
\adjustbox{max width=\textwidth}{%
\setlength{\tabcolsep}{8.25pt}
\begin{tabular}{@{}lccccccc@{}}
\toprule
\multicolumn{8}{c}{\textbf{Rotated MNIST}} \\
 &\multicolumn{3}{c}{\textbf{Target domains $\{0,30,60\}$}} &\multicolumn{3}{c}{\textbf{Target domains $\{15,45,75\}$}} & \\\midrule
Method   & 0    & 30   & 60    & 15    & 45   & 75                   & Avg                  \\
\hline
ERM & 96.0$\pm$0.3&   98.8$\pm$0.4  &  98.7$\pm$0.1  &  98.8$\pm$0.3   &   \textbf{99.1$\pm$0.1}  & 96.7$\pm$0.3 & 98.0  \\
IRM &  80.9$\pm$3.2   & 94.7$\pm$0.9 & 94.3$\pm$1.3 &  94.3$\pm$0.8   &   95.5$\pm$0.5  & 91.1$\pm$3.1 & 91.8  \\
\textbf{DRM} & \textbf{97.1$\pm$0.2}& \textbf{98.8$\pm$0.2} & \textbf{98.9$\pm$0.3} & \textbf{98.8$\pm$0.1} & 98.8$\pm$0.0 & \textbf{98.1$\pm$0.7} & \textbf{98.4 }\\
 \bottomrule
\end{tabular}}
\caption{Generalization performance on multiple unseen target domains.}
\vspace{-0.4cm}
\label{tab:multi-target}
\end{table*}

\subsection{Ablation Studies and Analysis}
\label{sec:analysis}

\textbf{Different model selection strategies.} Here we also conduct another baseline termed \textbf{Neural Network Measurement (NNM).} To fully utilize the modeling capability of the neural network, we propose estimating $\alpha_i\mathbb{E}_{{\mathcal{D}}_\T}[|f_i-f_\T|]$ by NN. Specifically, during training, a domain discriminator is trained to classify which domain is each image from. During test, for $x\in\D_\T$, the prediction result of the discriminator will be $\{d_i\}_{i=1}^K$, and $\{H_i=-d_i\}_{i=1}^K$ is used as the estimation.
We compare all the proposed strategies and a simple ensembling learning baseline, which uses a uniform weight for classifier ensembling.
Table~\ref{tab:test-time} (left) shows that the simple ensembling method works poorly in all domains.
In contrast, the proposed methods achieve consistent improvements and PEM generally performs best.

\textbf{Correlation matrix.} From the correlation matrices, we find that (i) the entropy of the predictions between one source domain and its corresponding classifier is minimal. 
(ii) In the target domain, the classifiers cannot attain a very low entropy as on the corresponding source domains.
(iii) The entropy of the predictions has a certain correlation with domain similarity. For example, in~\figurename~\ref{fig:heatmapsa}, the classifier for domain $d=1$ (with rotation angle $15^\circ$) achieves the minimum entropy in the unseen target domain $d=0$ (no rotation). As the rotation angle increases, the entropy also increases. This phenomenon also occurs in other domains. Refer to the appendix for more analysis.

\begin{figure*}[t]
    \subfigure[]{
    \includegraphics[width=0.23\linewidth]{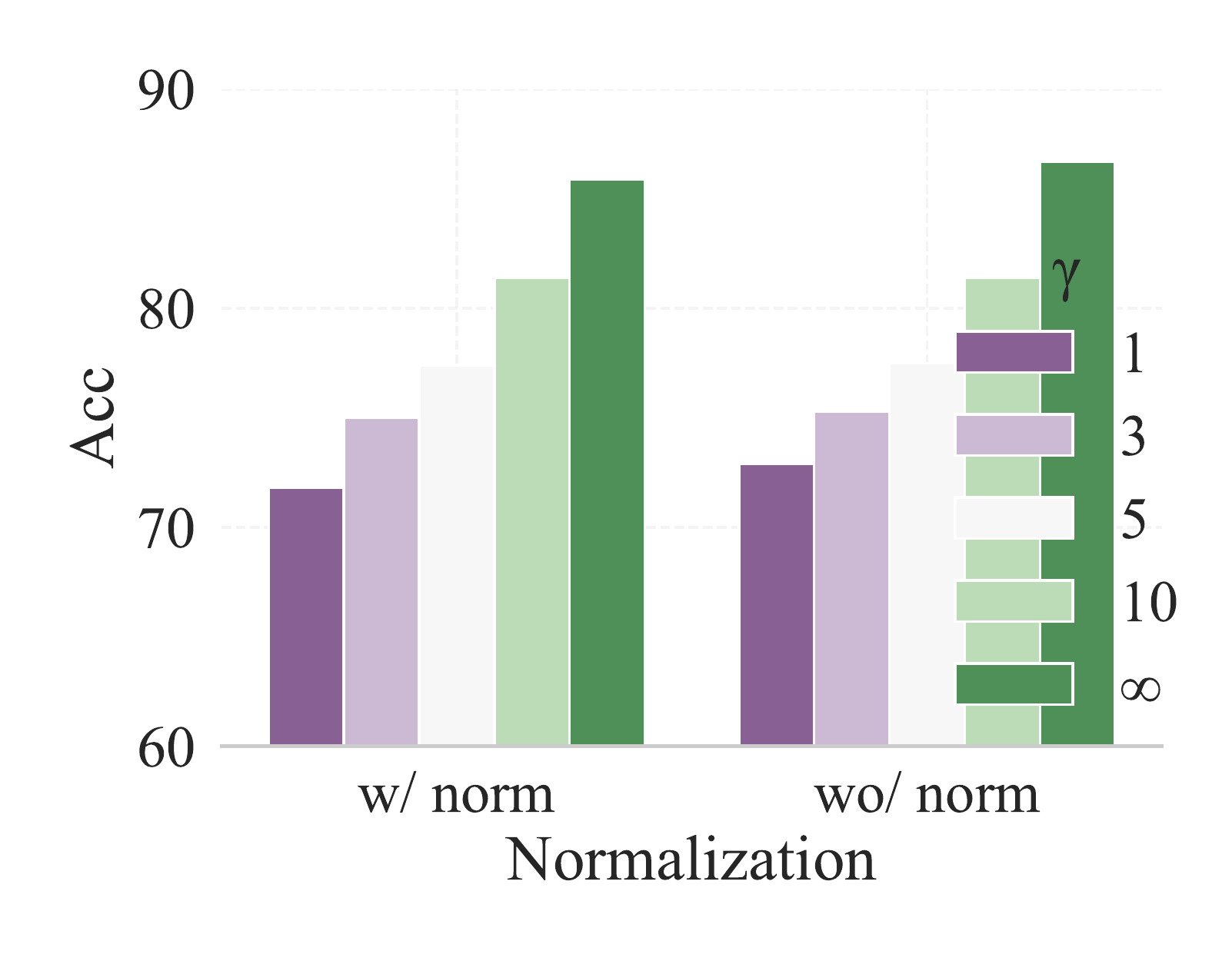}
    \vspace{-4mm}
    }
    \subfigure[]{
    \includegraphics[width=0.23\linewidth]{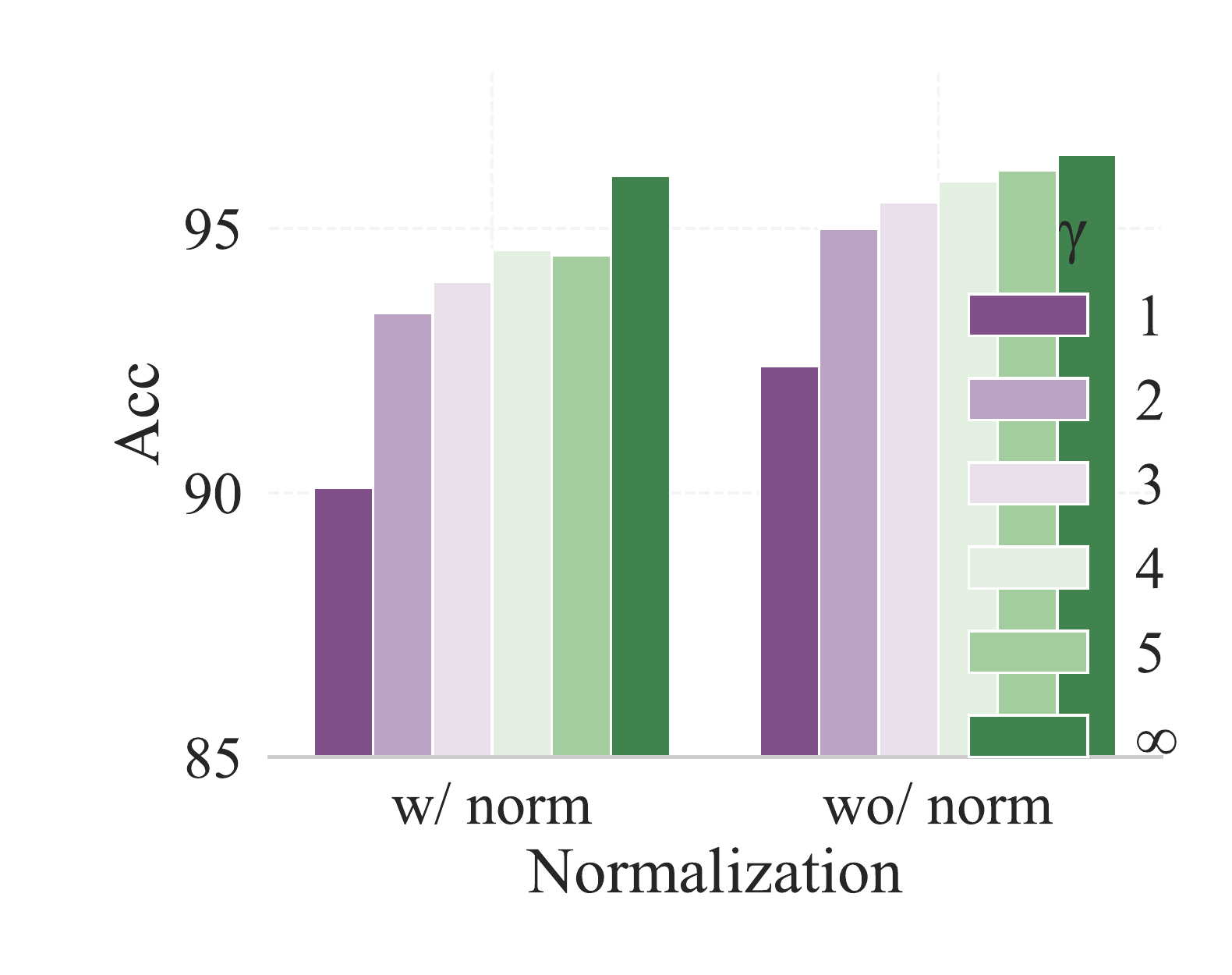}
    \vspace{-2mm}
    }
     \subfigure[]{
    \includegraphics[width=0.23\linewidth]{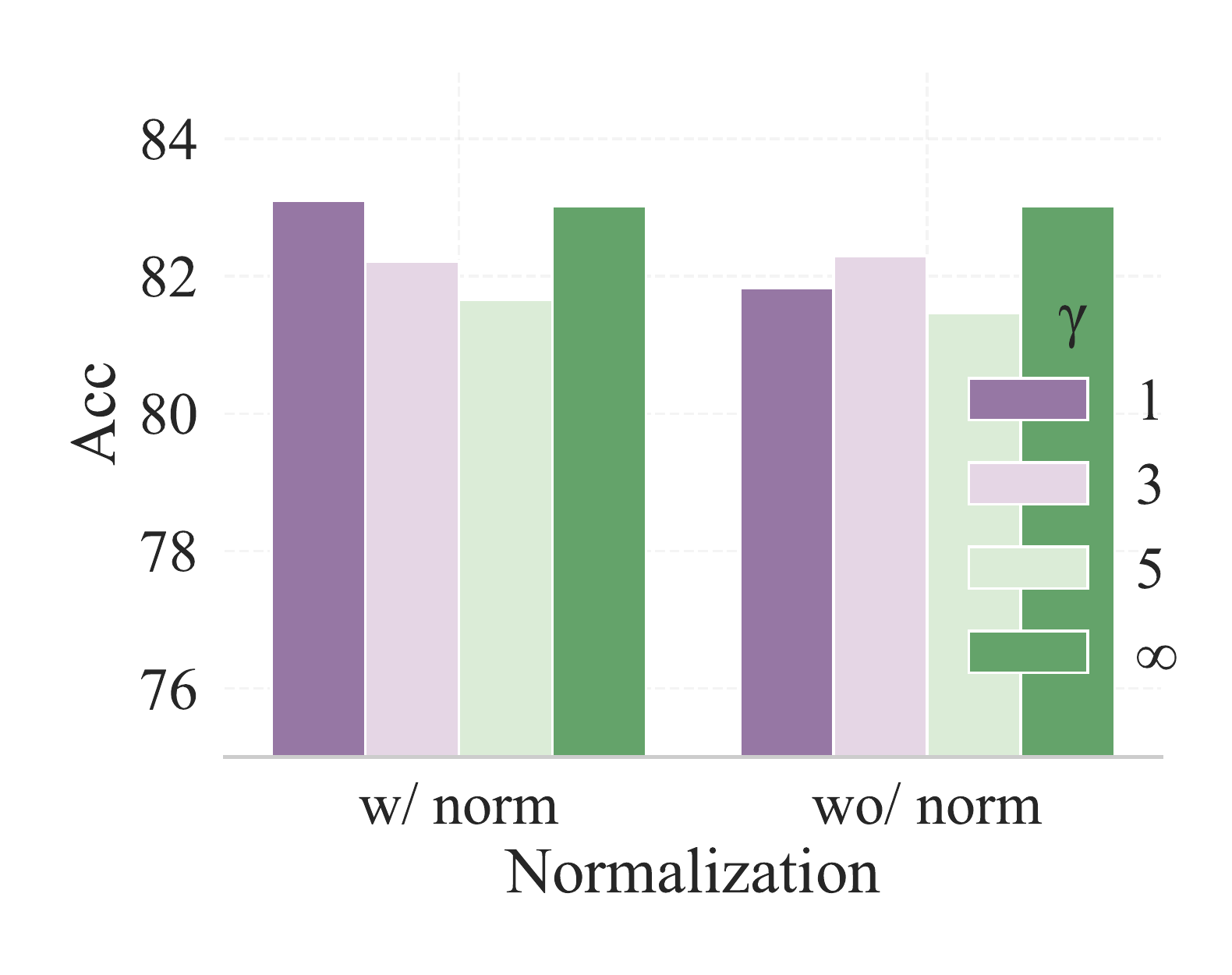}
    \vspace{-4mm}
    }
    \subfigure[]{
     \includegraphics[width=0.23\linewidth]{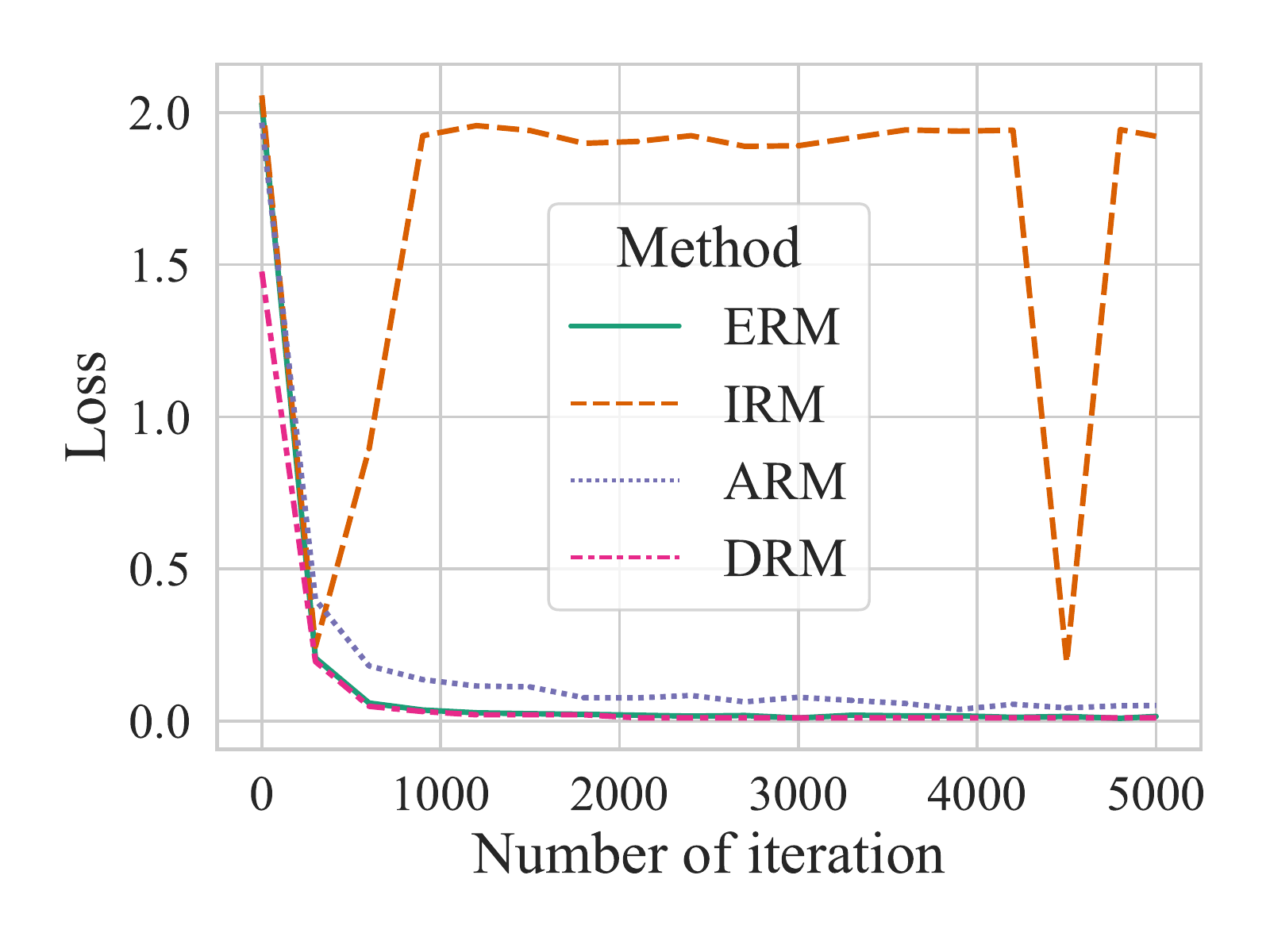}
     \label{fig:convergence}
    \vspace{-4mm}
    }
    \caption{Different mixing weights on the (a) Colored MNIST (target domain $d=2$) (b) Rotated MNIST (target domain $d=0$), and (c) PACS datasets (target domain $d=3$). Given a classification vector $\mathbf{\bar{y}}=[y_1,y_2,...,y_c]$, $c$ is the number of classes, performing normalization means that let $y_i=y_i/\sum_{j=1}^cy_j$ before mixing. (d) Loss curves of different baselines. }
    \label{fig:com_classifiers}
\end{figure*}
\begin{wraptable}{r}{3.3cm}
\vspace{-0.3cm}
\centering
\small
\begin{tabular}{@{}ccc@{}}
\toprule
Dataset & DRM & ERM \\ \midrule
CMNIST & 1.91 & 1.29 \\
RMNIST & 3.31 & 1.26 \\
PACS & 10.74 & 9.81 \\
VLCS & 10.74 & 8.64 \\
DomainNet & 11.15 & 9.34 \\ \bottomrule
\end{tabular}%
\caption{Comparison between inference times of one data sample in milliseconds.}
\label{tab:infer_tme}
\end{wraptable} 
\textbf{\abbr has comparable model complexity to existing DG methods.} As shown in Table~\ref{tab:test-time} (right), methods that require manipulating gradients (Fish~\cite{shi2022gradient}) or following the meta-learning pipeline (ARM~\cite{zhang2021adaptive}) have a much slower training speed compared to ERM. The proposed \abbr, without the need for aligning representations~\cite{ganin2016domain}, matching gradient~\cite{shi2022gradient}, or learning invariant representations~\cite{arjovsky2020invariant}, has a training speed that is faster than most existing DG methods, especially on small datasets RotatedMNIST. The training speed of \abbr is slower than ERM due to the training of additional classifiers. As the number of domains/classes increases or the feature dimension increases, the training time of \abbr will increase accordingly, 
however, \abbr is always comparable to ERM and much faster than Fish and ARM (Table ~\ref{tab:params_all}). For model parameters, since all classifiers in our implementation are just a linear layer, the total parameters of \abbr is similar to ERM and much less than existing methods such as CDANN and ARM.

\textbf{\abbr has comparable inference time to ERM.} The time cost of prediction for one data sample in the RotatedMNIST, PACS, VLCS, and DomainNet datasets are shown in Table.\ref{tab:infer_tme}. \abbr will not introduce significant computational overhead even on the DomainNet dataset, which has the most number of domains.

\textbf{Softing mixed weights}
\figurename~\ref{fig:com_classifiers} shows ablation experiments of the hyperparameter $\gamma$ on three benchmarks. Different benchmarks show different preferences on $\gamma$. For easy benchmarks, Rotated MNIST and Colored MNIST, softening mixed weights is needless. The reason behind this phenomenon can be found in \figurename~\ref{fig:heatmapsa}, the optimal classifier for the target domain $0$ of the Rotated MNIST is exactly the classifier $1$ and the prediction entropies will increase as the rotation angle increases. Hence, selecting the most approximate classifier based on the minimum entropy selection strategy is enough to attain superior generalization results. However, prediction entropies on other larger benchmarks, \eg VLCS, are not so regular as on the Rotated MNIST. On realistic benchmarks, a mixing of classifiers can bring some improvements. Besides, normalization, which is a method to reduce classification confidence\footnote{Given two classification results from 2 classifiers $[2.1, 0.4,0.5], [0.3,0.6,0.1]$ and assume the weights are all $1$. The result is $[2.4, 1.0, 0.6]$ with normalization and $[1.0, 0.73, 0.27]$ without normalization. The former is more confident than the latter.}, is also needless for semi-synthetic datasets (Rotated MNIST and Colored MNIST) and valuable for realistic benchmarks. 


\textbf{\abbr brings faster convergence speed.} The training dynamics of \abbr~and several baselines on PACS dataset are shown in~\figurename~\ref{fig:convergence}, where $d=0$ is the target domain. IRM is unstable and hard to converge. ARM follows a meta-learning pipeline and converges slowly. In contrast, \abbr converges even faster than ERM.
\vspace{-0.1cm}
\section{Concluding Remarks}\label{sec:conclusion}

We theoretically and empirically study the importance of the adaptivity gap for domain generalization.
Inspired by our theory, we propose a new domain generalization algorithm, \abbr to eliminate the negative effects brought by the adaptivity gap.
\abbr uses different classifier combinations for different test samples and beats existing DG methods and TTA methods by a large margin. 

Existing TTA methods for domain generalization need to adapt model parameters continually, therefore, the prediction behavior cannot be thoroughly tested in advance, causing some ethical concerns~\cite{iwasawa2021test}. \abbr alleviates this important issue because model retraining is not necessary. One potential drawback is the additional parameters incurred by the multi-classifiers structure, which can be reduced by advanced techniques and model designs, \eg varying coefficient technique~\cite{nie2021vcnet,hastie1993varying}. 

\section*{Acknowledgements}
This work was partially funded by the National Key R$\&$D Program of China (2022ZD0117901), and National Natural Science Foundation of China (62236010, and 62141608).

\bibliographystyle{ACM-Reference-Format}
\balance
\bibliography{sample-sigconf}


\begin{thebibliography}{69}


\ifx \showCODEN    \undefined \def \showCODEN     #1{\unskip}     \fi
\ifx \showDOI      \undefined \def \showDOI       #1{#1}\fi
\ifx \showISBNx    \undefined \def \showISBNx     #1{\unskip}     \fi
\ifx \showISBNxiii \undefined \def \showISBNxiii  #1{\unskip}     \fi
\ifx \showISSN     \undefined \def \showISSN      #1{\unskip}     \fi
\ifx \showLCCN     \undefined \def \showLCCN      #1{\unskip}     \fi
\ifx \shownote     \undefined \def \shownote      #1{#1}          \fi
\ifx \showarticletitle \undefined \def \showarticletitle #1{#1}   \fi
\ifx \showURL      \undefined \def \showURL       {\relax}        \fi
\providecommand\bibfield[2]{#2}
\providecommand\bibinfo[2]{#2}
\providecommand\natexlab[1]{#1}
\providecommand\showeprint[2][]{arXiv:#2}

\bibitem[Albuquerque et~al\mbox{.}(2019)]%
        {albuquerque2019generalizing}
\bibfield{author}{\bibinfo{person}{Isabela Albuquerque},
  \bibinfo{person}{Jo{\~a}o Monteiro}, \bibinfo{person}{Mohammad Darvishi},
  \bibinfo{person}{Tiago~H Falk}, {and} \bibinfo{person}{Ioannis Mitliagkas}.}
  \bibinfo{year}{2019}\natexlab{}.
\newblock \showarticletitle{Generalizing to unseen domains via distribution
  matching}.
\newblock \bibinfo{journal}{\emph{arXiv preprint arXiv:1911.00804}}
  (\bibinfo{year}{2019}).
\newblock


\bibitem[Albuquerque et~al\mbox{.}(2020)]%
        {albuquerque2020adversarial}
\bibfield{author}{\bibinfo{person}{Isabela Albuquerque},
  \bibinfo{person}{Jo{\~a}o Monteiro}, \bibinfo{person}{Mohammad Darvishi},
  \bibinfo{person}{Tiago~H Falk}, {and} \bibinfo{person}{Ioannis Mitliagkas}.}
  \bibinfo{year}{2020}\natexlab{}.
\newblock \showarticletitle{Adversarial target-invariant representation
  learning for domain generalization}. In \bibinfo{booktitle}{\emph{Arxiv}}.
\newblock


\bibitem[Arjovsky et~al\mbox{.}(2019)]%
        {arjovsky2020invariant}
\bibfield{author}{\bibinfo{person}{Martin Arjovsky}, \bibinfo{person}{L{\'e}on
  Bottou}, \bibinfo{person}{Ishaan Gulrajani}, {and} \bibinfo{person}{David
  Lopez-Paz}.} \bibinfo{year}{2019}\natexlab{}.
\newblock \showarticletitle{Invariant risk minimization}.
\newblock \bibinfo{journal}{\emph{arXiv preprint arXiv:1907.02893}}
  (\bibinfo{year}{2019}).
\newblock


\bibitem[Ben-David et~al\mbox{.}(2010)]%
        {ben2010theory}
\bibfield{author}{\bibinfo{person}{Shai Ben-David}, \bibinfo{person}{John
  Blitzer}, \bibinfo{person}{Koby Crammer}, \bibinfo{person}{Alex Kulesza},
  \bibinfo{person}{Fernando Pereira}, {and} \bibinfo{person}{Jennifer~Wortman
  Vaughan}.} \bibinfo{year}{2010}\natexlab{}.
\newblock \showarticletitle{A theory of learning from different domains}.
\newblock \bibinfo{journal}{\emph{Machine learning}} (\bibinfo{year}{2010}).
\newblock


\bibitem[Ben-David et~al\mbox{.}(2006)]%
        {2006Analysis}
\bibfield{author}{\bibinfo{person}{Shai Ben-David}, \bibinfo{person}{John
  Blitzer}, \bibinfo{person}{Koby Crammer}, {and} \bibinfo{person}{Fernando
  Pereira}.} \bibinfo{year}{2006}\natexlab{}.
\newblock \showarticletitle{Analysis of representations for domain adaptation}.
  In \bibinfo{booktitle}{\emph{NIPS}}.
\newblock


\bibitem[Ben-Tal et~al\mbox{.}(2009)]%
        {ben2009robust}
\bibfield{author}{\bibinfo{person}{Aharon Ben-Tal}, \bibinfo{person}{Laurent
  El~Ghaoui}, {and} \bibinfo{person}{Arkadi Nemirovski}.}
  \bibinfo{year}{2009}\natexlab{}.
\newblock \bibinfo{booktitle}{\emph{Robust Optimization}}.
\newblock \bibinfo{publisher}{Princeton university press}.
\newblock


\bibitem[Blanchard et~al\mbox{.}(2021)]%
        {blanchard2021domain}
\bibfield{author}{\bibinfo{person}{Gilles Blanchard},
  \bibinfo{person}{Aniket~Anand Deshmukh}, \bibinfo{person}{{\"U}r{\"u}n
  Dogan}, \bibinfo{person}{Gyemin Lee}, {and} \bibinfo{person}{Clayton Scott}.}
  \bibinfo{year}{2021}\natexlab{}.
\newblock \showarticletitle{Domain Generalization by Marginal Transfer
  Learning.}
\newblock \bibinfo{journal}{\emph{J. Mach. Learn. Res.}}
  (\bibinfo{year}{2021}).
\newblock


\bibitem[Chu et~al\mbox{.}(2022)]%
        {chu2022dna}
\bibfield{author}{\bibinfo{person}{Xu Chu}, \bibinfo{person}{Yujie Jin},
  \bibinfo{person}{Wenwu Zhu}, \bibinfo{person}{Yasha Wang},
  \bibinfo{person}{Xin Wang}, \bibinfo{person}{Shanghang Zhang}, {and}
  \bibinfo{person}{Hong Mei}.} \bibinfo{year}{2022}\natexlab{}.
\newblock \showarticletitle{DNA: Domain Generalization with Diversified Neural
  Averaging}. In \bibinfo{booktitle}{\emph{International Conference on Machine
  Learning}}. PMLR, \bibinfo{pages}{4010--4034}.
\newblock


\bibitem[Delage and Ye(2010)]%
        {delage2010distributionally}
\bibfield{author}{\bibinfo{person}{Erick Delage} {and} \bibinfo{person}{Yinyu
  Ye}.} \bibinfo{year}{2010}\natexlab{}.
\newblock \showarticletitle{Distributionally robust optimization under moment
  uncertainty with application to data-driven problems}.
\newblock \bibinfo{journal}{\emph{Operations research}} (\bibinfo{year}{2010}).
\newblock


\bibitem[Ding and Fu(2017)]%
        {ding2017deep}
\bibfield{author}{\bibinfo{person}{Zhengming Ding} {and} \bibinfo{person}{Yun
  Fu}.} \bibinfo{year}{2017}\natexlab{}.
\newblock \showarticletitle{Deep domain generalization with structured low-rank
  constraint}.
\newblock \bibinfo{journal}{\emph{IEEE Transactions on Image Processing}}
  \bibinfo{volume}{27}, \bibinfo{number}{1} (\bibinfo{year}{2017}),
  \bibinfo{pages}{304--313}.
\newblock


\bibitem[Domingos(1997)]%
        {domingos1997does}
\bibfield{author}{\bibinfo{person}{Pedro~M Domingos}.}
  \bibinfo{year}{1997}\natexlab{}.
\newblock \showarticletitle{Why Does Bagging Work? A Bayesian Account and its
  Implications.}. In \bibinfo{booktitle}{\emph{KDD}}. Citeseer,
  \bibinfo{pages}{155--158}.
\newblock


\bibitem[Dubey et~al\mbox{.}(2021)]%
        {dubey2021adaptive}
\bibfield{author}{\bibinfo{person}{Abhimanyu Dubey}, \bibinfo{person}{Vignesh
  Ramanathan}, \bibinfo{person}{Alex Pentland}, {and} \bibinfo{person}{Dhruv
  Mahajan}.} \bibinfo{year}{2021}\natexlab{}.
\newblock \showarticletitle{Adaptive methods for real-world domain
  generalization}. In \bibinfo{booktitle}{\emph{Proceedings of the IEEE/CVF
  Conference on Computer Vision and Pattern Recognition}}.
  \bibinfo{pages}{14340--14349}.
\newblock


\bibitem[Ganin et~al\mbox{.}(2016)]%
        {ganin2016domain}
\bibfield{author}{\bibinfo{person}{Yaroslav Ganin}, \bibinfo{person}{Evgeniya
  Ustinova}, \bibinfo{person}{Hana Ajakan}, \bibinfo{person}{Pascal Germain},
  \bibinfo{person}{Hugo Larochelle}, \bibinfo{person}{Fran{\c{c}}ois
  Laviolette}, \bibinfo{person}{Mario Marchand}, {and} \bibinfo{person}{Victor
  Lempitsky}.} \bibinfo{year}{2016}\natexlab{}.
\newblock \showarticletitle{Domain-adversarial training of neural networks}.
\newblock \bibinfo{journal}{\emph{The journal of machine learning research}}
  \bibinfo{volume}{17}, \bibinfo{number}{1} (\bibinfo{year}{2016}),
  \bibinfo{pages}{2096--2030}.
\newblock


\bibitem[Ghifary et~al\mbox{.}(2015)]%
        {ghifary2015domain}
\bibfield{author}{\bibinfo{person}{Muhammad Ghifary},
  \bibinfo{person}{W~Bastiaan Kleijn}, \bibinfo{person}{Mengjie Zhang}, {and}
  \bibinfo{person}{David Balduzzi}.} \bibinfo{year}{2015}\natexlab{}.
\newblock \showarticletitle{Domain generalization for object recognition with
  multi-task autoencoders}. In \bibinfo{booktitle}{\emph{ICCV}}.
\newblock


\bibitem[Gulrajani and Lopez-Paz(2021)]%
        {gulrajani2021in}
\bibfield{author}{\bibinfo{person}{Ishaan Gulrajani} {and}
  \bibinfo{person}{David Lopez-Paz}.} \bibinfo{year}{2021}\natexlab{}.
\newblock \showarticletitle{In Search of Lost Domain Generalization}. In
  \bibinfo{booktitle}{\emph{ICLR}}.
\newblock


\bibitem[Hastie and Tibshirani(1993)]%
        {hastie1993varying}
\bibfield{author}{\bibinfo{person}{Trevor Hastie} {and} \bibinfo{person}{Robert
  Tibshirani}.} \bibinfo{year}{1993}\natexlab{}.
\newblock \showarticletitle{Varying-coefficient models}.
\newblock \bibinfo{journal}{\emph{Journal of the Royal Statistical Society:
  Series B (Methodological)}} \bibinfo{volume}{55}, \bibinfo{number}{4}
  (\bibinfo{year}{1993}), \bibinfo{pages}{757--779}.
\newblock


\bibitem[Hu and Hong(2013)]%
        {hu2013kullback}
\bibfield{author}{\bibinfo{person}{Zhaolin Hu} {and} \bibinfo{person}{L~Jeff
  Hong}.} \bibinfo{year}{2013}\natexlab{}.
\newblock \showarticletitle{Kullback-Leibler divergence constrained
  distributionally robust optimization}.
\newblock \bibinfo{journal}{\emph{Available at Optimization Online}}
  (\bibinfo{year}{2013}).
\newblock


\bibitem[Huang et~al\mbox{.}(2020)]%
        {huang2020self}
\bibfield{author}{\bibinfo{person}{Zeyi Huang}, \bibinfo{person}{Haohan Wang},
  \bibinfo{person}{Eric~P Xing}, {and} \bibinfo{person}{Dong Huang}.}
  \bibinfo{year}{2020}\natexlab{}.
\newblock \showarticletitle{Self-challenging improves cross-domain
  generalization}. In \bibinfo{booktitle}{\emph{ECCV}}.
\newblock


\bibitem[Iwasawa and Matsuo(2021)]%
        {iwasawa2021test}
\bibfield{author}{\bibinfo{person}{Yusuke Iwasawa} {and}
  \bibinfo{person}{Yutaka Matsuo}.} \bibinfo{year}{2021}\natexlab{}.
\newblock \showarticletitle{Test-time classifier adjustment module for
  model-agnostic domain generalization}.
\newblock \bibinfo{journal}{\emph{Advances in Neural Information Processing
  Systems}}  \bibinfo{volume}{34} (\bibinfo{year}{2021}),
  \bibinfo{pages}{2427--2440}.
\newblock


\bibitem[Kpotufe and Martinet(2018)]%
        {kpotufe2018marginal}
\bibfield{author}{\bibinfo{person}{Samory Kpotufe} {and}
  \bibinfo{person}{Guillaume Martinet}.} \bibinfo{year}{2018}\natexlab{}.
\newblock \showarticletitle{Marginal singularity, and the benefits of labels in
  covariate-shift}. In \bibinfo{booktitle}{\emph{Conference On Learning
  Theory}}. PMLR, \bibinfo{pages}{1882--1886}.
\newblock


\bibitem[Krueger et~al\mbox{.}(2021)]%
        {krueger2021out}
\bibfield{author}{\bibinfo{person}{David Krueger}, \bibinfo{person}{Ethan
  Caballero}, \bibinfo{person}{Joern-Henrik Jacobsen}, \bibinfo{person}{Amy
  Zhang}, \bibinfo{person}{Jonathan Binas}, \bibinfo{person}{Dinghuai Zhang},
  \bibinfo{person}{Remi Le~Priol}, {and} \bibinfo{person}{Aaron Courville}.}
  \bibinfo{year}{2021}\natexlab{}.
\newblock \showarticletitle{Out-of-distribution generalization via risk
  extrapolation (rex)}. In \bibinfo{booktitle}{\emph{ICML}}.
\newblock


\bibitem[Lee et~al\mbox{.}(2013)]%
        {lee2013pseudo}
\bibfield{author}{\bibinfo{person}{Dong-Hyun Lee} {et~al\mbox{.}}}
  \bibinfo{year}{2013}\natexlab{}.
\newblock \showarticletitle{Pseudo-label: The simple and efficient
  semi-supervised learning method for deep neural networks}. In
  \bibinfo{booktitle}{\emph{Workshop on challenges in representation learning,
  ICML}}.
\newblock


\bibitem[Li et~al\mbox{.}(2018c)]%
        {li2018learning}
\bibfield{author}{\bibinfo{person}{Da Li}, \bibinfo{person}{Yongxin Yang},
  \bibinfo{person}{Yi-Zhe Song}, {and} \bibinfo{person}{Timothy Hospedales}.}
  \bibinfo{year}{2018}\natexlab{c}.
\newblock \showarticletitle{Learning to generalize: Meta-learning for domain
  generalization}. In \bibinfo{booktitle}{\emph{AAAI}}.
\newblock


\bibitem[Li et~al\mbox{.}(2017)]%
        {li2017deeper}
\bibfield{author}{\bibinfo{person}{Da Li}, \bibinfo{person}{Yongxin Yang},
  \bibinfo{person}{Yi-Zhe Song}, {and} \bibinfo{person}{Timothy~M Hospedales}.}
  \bibinfo{year}{2017}\natexlab{}.
\newblock \showarticletitle{Deeper, broader and artier domain generalization}.
  In \bibinfo{booktitle}{\emph{ICCV}}.
\newblock


\bibitem[Li et~al\mbox{.}(2018a)]%
        {li2018domain}
\bibfield{author}{\bibinfo{person}{Haoliang Li}, \bibinfo{person}{Sinno~Jialin
  Pan}, \bibinfo{person}{Shiqi Wang}, {and} \bibinfo{person}{Alex~C Kot}.}
  \bibinfo{year}{2018}\natexlab{a}.
\newblock \showarticletitle{Domain generalization with adversarial feature
  learning}. In \bibinfo{booktitle}{\emph{CVPR}}.
\newblock


\bibitem[Li et~al\mbox{.}(2018b)]%
        {li2018deep}
\bibfield{author}{\bibinfo{person}{Ya Li}, \bibinfo{person}{Xinmei Tian},
  \bibinfo{person}{Mingming Gong}, \bibinfo{person}{Yajing Liu},
  \bibinfo{person}{Tongliang Liu}, \bibinfo{person}{Kun Zhang}, {and}
  \bibinfo{person}{Dacheng Tao}.} \bibinfo{year}{2018}\natexlab{b}.
\newblock \showarticletitle{Deep domain generalization via conditional
  invariant adversarial networks}. In \bibinfo{booktitle}{\emph{ECCV}}.
\newblock


\bibitem[Liang et~al\mbox{.}(2023)]%
        {liang2023comprehensive}
\bibfield{author}{\bibinfo{person}{Jian Liang}, \bibinfo{person}{Ran He}, {and}
  \bibinfo{person}{Tieniu Tan}.} \bibinfo{year}{2023}\natexlab{}.
\newblock \showarticletitle{A Comprehensive Survey on Test-Time Adaptation
  under Distribution Shifts}.
\newblock \bibinfo{journal}{\emph{arXiv preprint arXiv:2303.15361}}
  (\bibinfo{year}{2023}).
\newblock


\bibitem[Liang et~al\mbox{.}(2020)]%
        {liang2020we}
\bibfield{author}{\bibinfo{person}{Jian Liang}, \bibinfo{person}{Dapeng Hu},
  {and} \bibinfo{person}{Jiashi Feng}.} \bibinfo{year}{2020}\natexlab{}.
\newblock \showarticletitle{Do we really need to access the source data? source
  hypothesis transfer for unsupervised domain adaptation}. In
  \bibinfo{booktitle}{\emph{International Conference on Machine Learning}}.
  PMLR, \bibinfo{pages}{6028--6039}.
\newblock


\bibitem[Liu et~al\mbox{.}(2021c)]%
        {liu2021learning}
\bibfield{author}{\bibinfo{person}{Chang Liu}, \bibinfo{person}{Xinwei Sun},
  \bibinfo{person}{Jindong Wang}, \bibinfo{person}{Haoyue Tang},
  \bibinfo{person}{Tao Li}, \bibinfo{person}{Tao Qin}, \bibinfo{person}{Wei
  Chen}, {and} \bibinfo{person}{Tie-Yan Liu}.}
  \bibinfo{year}{2021}\natexlab{c}.
\newblock \showarticletitle{Learning causal semantic representation for
  out-of-distribution prediction}.
\newblock \bibinfo{journal}{\emph{Advances in Neural Information Processing
  Systems}}  \bibinfo{volume}{34} (\bibinfo{year}{2021}),
  \bibinfo{pages}{6155--6170}.
\newblock


\bibitem[Liu et~al\mbox{.}(2021a)]%
        {liu2021just}
\bibfield{author}{\bibinfo{person}{Evan~Z Liu}, \bibinfo{person}{Behzad
  Haghgoo}, \bibinfo{person}{Annie~S Chen}, \bibinfo{person}{Aditi
  Raghunathan}, \bibinfo{person}{Pang~Wei Koh}, \bibinfo{person}{Shiori
  Sagawa}, \bibinfo{person}{Percy Liang}, {and} \bibinfo{person}{Chelsea
  Finn}.} \bibinfo{year}{2021}\natexlab{a}.
\newblock \showarticletitle{Just Train Twice: Improving Group Robustness
  without Training Group Information}. In
  \bibinfo{booktitle}{\emph{International Conference on Machine Learning
  (ICML)}}.
\newblock


\bibitem[Liu et~al\mbox{.}(2021b)]%
        {liu2021domain}
\bibfield{author}{\bibinfo{person}{Xiaofeng Liu}, \bibinfo{person}{Bo Hu},
  \bibinfo{person}{Linghao Jin}, \bibinfo{person}{Xu Han},
  \bibinfo{person}{Fangxu Xing}, \bibinfo{person}{Jinsong Ouyang},
  \bibinfo{person}{Jun Lu}, \bibinfo{person}{Georges~EL Fakhri}, {and}
  \bibinfo{person}{Jonghye Woo}.} \bibinfo{year}{2021}\natexlab{b}.
\newblock \showarticletitle{Domain generalization under conditional and label
  shifts via variational bayesian inference}.
\newblock \bibinfo{journal}{\emph{arXiv preprint arXiv:2107.10931}}
  (\bibinfo{year}{2021}).
\newblock


\bibitem[Lu et~al\mbox{.}(2022)]%
        {lu2022domain}
\bibfield{author}{\bibinfo{person}{Wang Lu}, \bibinfo{person}{Jindong Wang},
  \bibinfo{person}{Haoliang Li}, \bibinfo{person}{Yiqiang Chen}, {and}
  \bibinfo{person}{Xing Xie}.} \bibinfo{year}{2022}\natexlab{}.
\newblock \showarticletitle{Domain-invariant Feature Exploration for Domain
  Generalization}.
\newblock \bibinfo{journal}{\emph{Transactions on Machine Learning Research
  (TMLR)}} (\bibinfo{year}{2022}).
\newblock


\bibitem[Lu et~al\mbox{.}(2023)]%
        {lu2023out}
\bibfield{author}{\bibinfo{person}{Wang Lu}, \bibinfo{person}{Jindong Wang},
  \bibinfo{person}{Xinwei Sun}, \bibinfo{person}{Yiqiang Chen}, {and}
  \bibinfo{person}{Xing Xie}.} \bibinfo{year}{2023}\natexlab{}.
\newblock \showarticletitle{Out-of-distribution Representation Learning for
  Time Series Classification}. In \bibinfo{booktitle}{\emph{International
  Conference on Learning Representations (ICLR)}}.
\newblock


\bibitem[Michel et~al\mbox{.}(2021)]%
        {michel2021modeling}
\bibfield{author}{\bibinfo{person}{Paul Michel}, \bibinfo{person}{Tatsunori
  Hashimoto}, {and} \bibinfo{person}{Graham Neubig}.}
  \bibinfo{year}{2021}\natexlab{}.
\newblock \showarticletitle{Modeling the Second Player in Distributionally
  Robust Optimization}. In \bibinfo{booktitle}{\emph{International Conference
  on Learning Representations (ICLR)}}.
\newblock


\bibitem[Muandet et~al\mbox{.}(2013)]%
        {MuaBalSch13}
\bibfield{author}{\bibinfo{person}{K. Muandet}, \bibinfo{person}{D. Balduzzi},
  {and} \bibinfo{person}{B. Sch{\"o}lkopf}.} \bibinfo{year}{2013}\natexlab{}.
\newblock \showarticletitle{Domain Generalization via Invariant Feature
  Representation}. In \bibinfo{booktitle}{\emph{ICML}}.
\newblock


\bibitem[Nam et~al\mbox{.}(2021)]%
        {nam2021reducing}
\bibfield{author}{\bibinfo{person}{Hyeonseob Nam}, \bibinfo{person}{HyunJae
  Lee}, \bibinfo{person}{Jongchan Park}, \bibinfo{person}{Wonjun Yoon}, {and}
  \bibinfo{person}{Donggeun Yoo}.} \bibinfo{year}{2021}\natexlab{}.
\newblock \showarticletitle{Reducing Domain Gap by Reducing Style Bias}. In
  \bibinfo{booktitle}{\emph{CVPR}}.
\newblock


\bibitem[Nie et~al\mbox{.}(2020)]%
        {nie2021vcnet}
\bibfield{author}{\bibinfo{person}{Lizhen Nie}, \bibinfo{person}{Mao Ye},
  \bibinfo{person}{Qiang Liu}, {and} \bibinfo{person}{Dan Nicolae}.}
  \bibinfo{year}{2020}\natexlab{}.
\newblock \showarticletitle{Vcnet and functional targeted regularization for
  learning causal effects of continuous treatments}.
\newblock \bibinfo{journal}{\emph{ICLR}} (\bibinfo{year}{2020}).
\newblock


\bibitem[Nowozin et~al\mbox{.}(2016)]%
        {nowozin2016f}
\bibfield{author}{\bibinfo{person}{Sebastian Nowozin}, \bibinfo{person}{Botond
  Cseke}, {and} \bibinfo{person}{Ryota Tomioka}.}
  \bibinfo{year}{2016}\natexlab{}.
\newblock \showarticletitle{f-gan: Training generative neural samplers using
  variational divergence minimization}.
\newblock \bibinfo{journal}{\emph{Advances in neural information processing
  systems}}  \bibinfo{volume}{29} (\bibinfo{year}{2016}).
\newblock


\bibitem[Oh et~al\mbox{.}(2022)]%
        {oh2022learning}
\bibfield{author}{\bibinfo{person}{Changdae Oh}, \bibinfo{person}{Heeji Won},
  \bibinfo{person}{Junhyuk So}, \bibinfo{person}{Taero Kim},
  \bibinfo{person}{Yewon Kim}, \bibinfo{person}{Hosik Choi}, {and}
  \bibinfo{person}{Kyungwoo Song}.} \bibinfo{year}{2022}\natexlab{}.
\newblock \showarticletitle{Learning Fair Representation via Distributional
  Contrastive Disentanglement}. In \bibinfo{booktitle}{\emph{Proceedings of the
  28th ACM SIGKDD Conference on Knowledge Discovery and Data Mining}}.
  \bibinfo{pages}{1295--1305}.
\newblock


\bibitem[Peng et~al\mbox{.}(2019)]%
        {peng2019moment}
\bibfield{author}{\bibinfo{person}{Xingchao Peng}, \bibinfo{person}{Qinxun
  Bai}, \bibinfo{person}{Xide Xia}, \bibinfo{person}{Zijun Huang},
  \bibinfo{person}{Kate Saenko}, {and} \bibinfo{person}{Bo Wang}.}
  \bibinfo{year}{2019}\natexlab{}.
\newblock \showarticletitle{Moment matching for multi-source domain
  adaptation}. In \bibinfo{booktitle}{\emph{ICCV}}.
\newblock


\bibitem[Rame et~al\mbox{.}(2022)]%
        {rame2021fishr}
\bibfield{author}{\bibinfo{person}{Alexandre Rame}, \bibinfo{person}{Corentin
  Dancette}, {and} \bibinfo{person}{Matthieu Cord}.}
  \bibinfo{year}{2022}\natexlab{}.
\newblock \showarticletitle{Fishr: Invariant gradient variances for
  out-of-distribution generalization}. In
  \bibinfo{booktitle}{\emph{International Conference on Machine Learning}}.
  PMLR, \bibinfo{pages}{18347--18377}.
\newblock


\bibitem[Sagawa et~al\mbox{.}(2020)]%
        {sagawa2020distributionally}
\bibfield{author}{\bibinfo{person}{Shiori Sagawa}, \bibinfo{person}{Pang~Wei
  Koh}, \bibinfo{person}{Tatsunori~B Hashimoto}, {and} \bibinfo{person}{Percy
  Liang}.} \bibinfo{year}{2020}\natexlab{}.
\newblock \showarticletitle{Distributionally robust neural networks for group
  shifts: On the importance of regularization for worst-case generalization}.
  In \bibinfo{booktitle}{\emph{International conference on learning
  representations (ICLR)}}.
\newblock


\bibitem[Segu et~al\mbox{.}(2020)]%
        {segu2020batch}
\bibfield{author}{\bibinfo{person}{Mattia Segu}, \bibinfo{person}{Alessio
  Tonioni}, {and} \bibinfo{person}{Federico Tombari}.}
  \bibinfo{year}{2020}\natexlab{}.
\newblock \showarticletitle{Batch normalization embeddings for deep domain
  generalization}.
\newblock \bibinfo{journal}{\emph{arXiv preprint arXiv:2011.12672}}
  (\bibinfo{year}{2020}).
\newblock


\bibitem[Shi et~al\mbox{.}(2022b)]%
        {shi2022pairwise}
\bibfield{author}{\bibinfo{person}{Weili Shi}, \bibinfo{person}{Ronghang Zhu},
  {and} \bibinfo{person}{Sheng Li}.} \bibinfo{year}{2022}\natexlab{b}.
\newblock \showarticletitle{Pairwise Adversarial Training for Unsupervised
  Class-imbalanced Domain Adaptation}. In \bibinfo{booktitle}{\emph{Proceedings
  of the 28th ACM SIGKDD Conference on Knowledge Discovery and Data Mining}}.
  \bibinfo{pages}{1598--1606}.
\newblock


\bibitem[Shi et~al\mbox{.}(2022a)]%
        {shi2022gradient}
\bibfield{author}{\bibinfo{person}{Yuge Shi}, \bibinfo{person}{Jeffrey Seely},
  \bibinfo{person}{Philip Torr}, \bibinfo{person}{Siddharth N},
  \bibinfo{person}{Awni Hannun}, \bibinfo{person}{Nicolas Usunier}, {and}
  \bibinfo{person}{Gabriel Synnaeve}.} \bibinfo{year}{2022}\natexlab{a}.
\newblock \showarticletitle{Gradient Matching for Domain Generalization}. In
  \bibinfo{booktitle}{\emph{International Conference on Learning
  Representations}}.
\newblock
\urldef\tempurl%
\url{https://openreview.net/forum?id=vDwBW49HmO}
\showURL{%
\tempurl}


\bibitem[Sinha et~al\mbox{.}(2017)]%
        {sinha2017certifying}
\bibfield{author}{\bibinfo{person}{Aman Sinha}, \bibinfo{person}{Hongseok
  Namkoong}, \bibinfo{person}{Riccardo Volpi}, {and} \bibinfo{person}{John
  Duchi}.} \bibinfo{year}{2017}\natexlab{}.
\newblock \showarticletitle{Certifying some distributional robustness with
  principled adversarial training}.
\newblock \bibinfo{journal}{\emph{arXiv preprint arXiv:1710.10571}}
  (\bibinfo{year}{2017}).
\newblock


\bibitem[Staib and Jegelka(2019)]%
        {staib2019distributionally}
\bibfield{author}{\bibinfo{person}{Matthew Staib} {and}
  \bibinfo{person}{Stefanie Jegelka}.} \bibinfo{year}{2019}\natexlab{}.
\newblock \showarticletitle{Distributionally robust optimization and
  generalization in kernel methods}.
\newblock \bibinfo{journal}{\emph{Advances in Neural Information Processing
  Systems (NeurIPS)}} (\bibinfo{year}{2019}).
\newblock


\bibitem[Stojanov et~al\mbox{.}(2021)]%
        {stojanov2021domain}
\bibfield{author}{\bibinfo{person}{Petar Stojanov}, \bibinfo{person}{Zijian
  Li}, \bibinfo{person}{Mingming Gong}, \bibinfo{person}{Ruichu Cai},
  \bibinfo{person}{Jaime~G. Carbonell}, {and} \bibinfo{person}{Kun Zhang}.}
  \bibinfo{year}{2021}\natexlab{}.
\newblock \showarticletitle{Domain Adaptation with Invariant Representation
  Learning: What Transformations to Learn?}. In
  \bibinfo{booktitle}{\emph{NeurIPS}}.
\newblock


\bibitem[Sun and Saenko(2016)]%
        {sun2016deep}
\bibfield{author}{\bibinfo{person}{Baochen Sun} {and} \bibinfo{person}{Kate
  Saenko}.} \bibinfo{year}{2016}\natexlab{}.
\newblock \showarticletitle{Deep coral: Correlation alignment for deep domain
  adaptation}. In \bibinfo{booktitle}{\emph{ECCV}}.
\newblock


\bibitem[Sun et~al\mbox{.}(2020)]%
        {sun2020test}
\bibfield{author}{\bibinfo{person}{Yu Sun}, \bibinfo{person}{Xiaolong Wang},
  \bibinfo{person}{Zhuang Liu}, \bibinfo{person}{John Miller},
  \bibinfo{person}{Alexei Efros}, {and} \bibinfo{person}{Moritz Hardt}.}
  \bibinfo{year}{2020}\natexlab{}.
\newblock \showarticletitle{Test-time training with self-supervision for
  generalization under distribution shifts}. In
  \bibinfo{booktitle}{\emph{International conference on machine learning}}.
  PMLR, \bibinfo{pages}{9229--9248}.
\newblock


\bibitem[Teney et~al\mbox{.}(2022)]%
        {teney2022id}
\bibfield{author}{\bibinfo{person}{Damien Teney}, \bibinfo{person}{Seong~Joon
  Oh}, {and} \bibinfo{person}{Ehsan Abbasnejad}.}
  \bibinfo{year}{2022}\natexlab{}.
\newblock \showarticletitle{ID and OOD Performance Are Sometimes Inversely
  Correlated on Real-world Datasets}.
\newblock \bibinfo{journal}{\emph{arXiv preprint arXiv:2209.00613}}
  (\bibinfo{year}{2022}).
\newblock


\bibitem[Torralba and Efros(2011)]%
        {torralba2011unbiased}
\bibfield{author}{\bibinfo{person}{Antonio Torralba} {and}
  \bibinfo{person}{Alexei~A Efros}.} \bibinfo{year}{2011}\natexlab{}.
\newblock \showarticletitle{Unbiased look at dataset bias}. In
  \bibinfo{booktitle}{\emph{CVPR}}.
\newblock


\bibitem[Vapnik(1999)]%
        {vapnik1998statistical}
\bibfield{author}{\bibinfo{person}{Vladimir Vapnik}.}
  \bibinfo{year}{1999}\natexlab{}.
\newblock \bibinfo{booktitle}{\emph{The nature of statistical learning
  theory}}.
\newblock \bibinfo{publisher}{Springer science \& business media}.
\newblock


\bibitem[Wang et~al\mbox{.}(2021)]%
        {wang2020tent}
\bibfield{author}{\bibinfo{person}{Dequan Wang}, \bibinfo{person}{Evan
  Shelhamer}, \bibinfo{person}{Shaoteng Liu}, \bibinfo{person}{Bruno
  Olshausen}, {and} \bibinfo{person}{Trevor Darrell}.}
  \bibinfo{year}{2021}\natexlab{}.
\newblock \showarticletitle{Tent: Fully Test-Time Adaptation by Entropy
  Minimization}. In \bibinfo{booktitle}{\emph{ICLR}}.
\newblock


\bibitem[Wang et~al\mbox{.}(2022)]%
        {wang2022generalizing}
\bibfield{author}{\bibinfo{person}{Jindong Wang}, \bibinfo{person}{Cuiling
  Lan}, \bibinfo{person}{Chang Liu}, \bibinfo{person}{Yidong Ouyang},
  \bibinfo{person}{Tao Qin}, \bibinfo{person}{Wang Lu},
  \bibinfo{person}{Yiqiang Chen}, \bibinfo{person}{Wenjun Zeng}, {and}
  \bibinfo{person}{Philip Yu}.} \bibinfo{year}{2022}\natexlab{}.
\newblock \showarticletitle{Generalizing to unseen domains: A survey on domain
  generalization}.
\newblock \bibinfo{journal}{\emph{IEEE Transactions on Knowledge and Data
  Engineering}} (\bibinfo{year}{2022}).
\newblock


\bibitem[Wang et~al\mbox{.}(2020)]%
        {wang2020dofe}
\bibfield{author}{\bibinfo{person}{Shujun Wang}, \bibinfo{person}{Lequan Yu},
  \bibinfo{person}{Kang Li}, \bibinfo{person}{Xin Yang},
  \bibinfo{person}{Chi-Wing Fu}, {and} \bibinfo{person}{Pheng-Ann Heng}.}
  \bibinfo{year}{2020}\natexlab{}.
\newblock \showarticletitle{Dofe: Domain-oriented feature embedding for
  generalizable fundus image segmentation on unseen datasets}.
\newblock \bibinfo{journal}{\emph{IEEE Transactions on Medical Imaging}}
  \bibinfo{volume}{39}, \bibinfo{number}{12} (\bibinfo{year}{2020}),
  \bibinfo{pages}{4237--4248}.
\newblock


\bibitem[Yan et~al\mbox{.}(2020)]%
        {yan2020improve}
\bibfield{author}{\bibinfo{person}{Shen Yan}, \bibinfo{person}{Huan Song},
  \bibinfo{person}{Nanxiang Li}, \bibinfo{person}{Lincan Zou}, {and}
  \bibinfo{person}{Liu Ren}.} \bibinfo{year}{2020}\natexlab{}.
\newblock \showarticletitle{Improve unsupervised domain adaptation with mixup
  training}.
\newblock \bibinfo{journal}{\emph{arXiv preprint arXiv:2001.00677}}
  (\bibinfo{year}{2020}).
\newblock


\bibitem[Yang et~al\mbox{.}(2021)]%
        {yang2021generalized}
\bibfield{author}{\bibinfo{person}{Shiqi Yang}, \bibinfo{person}{Yaxing Wang},
  \bibinfo{person}{Joost van~de Weijer}, \bibinfo{person}{Luis Herranz}, {and}
  \bibinfo{person}{Shangling Jui}.} \bibinfo{year}{2021}\natexlab{}.
\newblock \showarticletitle{Generalized source-free domain adaptation}. In
  \bibinfo{booktitle}{\emph{Proceedings of the IEEE/CVF International
  Conference on Computer Vision}}. \bibinfo{pages}{8978--8987}.
\newblock


\bibitem[Ye et~al\mbox{.}(2022)]%
        {ye2022ood}
\bibfield{author}{\bibinfo{person}{Nanyang Ye}, \bibinfo{person}{Kaican Li},
  \bibinfo{person}{Haoyue Bai}, \bibinfo{person}{Runpeng Yu},
  \bibinfo{person}{Lanqing Hong}, \bibinfo{person}{Fengwei Zhou},
  \bibinfo{person}{Zhenguo Li}, {and} \bibinfo{person}{Jun Zhu}.}
  \bibinfo{year}{2022}\natexlab{}.
\newblock \showarticletitle{OoD-Bench: Quantifying and Understanding Two
  Dimensions of Out-of-Distribution Generalization}. In
  \bibinfo{booktitle}{\emph{Proceedings of the IEEE/CVF Conference on Computer
  Vision and Pattern Recognition}}. \bibinfo{pages}{7947--7958}.
\newblock


\bibitem[Zhang et~al\mbox{.}(2022d)]%
        {zhang2021towards}
\bibfield{author}{\bibinfo{person}{Hanlin Zhang}, \bibinfo{person}{Yi-Fan
  Zhang}, \bibinfo{person}{Weiyang Liu}, \bibinfo{person}{Adrian Weller},
  \bibinfo{person}{Bernhard Sch{\"o}lkopf}, {and} \bibinfo{person}{Eric~P
  Xing}.} \bibinfo{year}{2022}\natexlab{d}.
\newblock \showarticletitle{Towards principled disentanglement for domain
  generalization}. In \bibinfo{booktitle}{\emph{Proceedings of the IEEE/CVF
  Conference on Computer Vision and Pattern Recognition}}.
\newblock


\bibitem[Zhang et~al\mbox{.}(2013)]%
        {zhang2013domain}
\bibfield{author}{\bibinfo{person}{Kun Zhang}, \bibinfo{person}{Bernhard
  Sch{\"o}lkopf}, \bibinfo{person}{Krikamol Muandet}, {and}
  \bibinfo{person}{Zhikun Wang}.} \bibinfo{year}{2013}\natexlab{}.
\newblock \showarticletitle{Domain adaptation under target and conditional
  shift}. In \bibinfo{booktitle}{\emph{International Conference on Machine
  Learning}}. PMLR, \bibinfo{pages}{819--827}.
\newblock


\bibitem[Zhang et~al\mbox{.}(2021b)]%
        {zhang2021adaptive}
\bibfield{author}{\bibinfo{person}{Marvin Zhang}, \bibinfo{person}{Henrik
  Marklund}, \bibinfo{person}{Nikita Dhawan}, \bibinfo{person}{Abhishek Gupta},
  \bibinfo{person}{Sergey Levine}, {and} \bibinfo{person}{Chelsea Finn}.}
  \bibinfo{year}{2021}\natexlab{b}.
\newblock \showarticletitle{Adaptive risk minimization: Learning to adapt to
  domain shift}.
\newblock \bibinfo{journal}{\emph{NeurIPS}} (\bibinfo{year}{2021}).
\newblock


\bibitem[Zhang et~al\mbox{.}(2021a)]%
        {zhang2021test}
\bibfield{author}{\bibinfo{person}{Yifan Zhang}, \bibinfo{person}{Bryan Hooi},
  \bibinfo{person}{Lanqing Hong}, {and} \bibinfo{person}{Jiashi Feng}.}
  \bibinfo{year}{2021}\natexlab{a}.
\newblock \showarticletitle{Test-agnostic long-tailed recognition by test-time
  aggregating diverse experts with self-supervision}.
\newblock \bibinfo{journal}{\emph{arXiv preprint arXiv:2107.09249}}
  (\bibinfo{year}{2021}).
\newblock


\bibitem[Zhang et~al\mbox{.}(2022a)]%
        {zhang2022generalizable}
\bibfield{author}{\bibinfo{person}{YiFan Zhang}, \bibinfo{person}{Feng Li},
  \bibinfo{person}{Zhang Zhang}, \bibinfo{person}{Liang Wang},
  \bibinfo{person}{Dacheng Tao}, {and} \bibinfo{person}{Tieniu Tan}.}
  \bibinfo{year}{2022}\natexlab{a}.
\newblock \bibinfo{title}{Generalizable Person Re-identification Without
  Demographics}.
\newblock
\newblock
\urldef\tempurl%
\url{https://openreview.net/forum?id=VNdFPD5wqjh}
\showURL{%
\tempurl}


\bibitem[Zhang et~al\mbox{.}(2023b)]%
        {zhang2023free}
\bibfield{author}{\bibinfo{person}{YiFan Zhang}, \bibinfo{person}{Xue Wang},
  \bibinfo{person}{Jian Liang}, \bibinfo{person}{Zhang Zhang},
  \bibinfo{person}{Liang Wang}, \bibinfo{person}{Rong Jin}, {and}
  \bibinfo{person}{Tieniu Tan}.} \bibinfo{year}{2023}\natexlab{b}.
\newblock \showarticletitle{Free Lunch for Domain Adversarial Training:
  Environment Label Smoothing}.
\newblock \bibinfo{journal}{\emph{International Conference on Learning
  Representations (ICLR)}} (\bibinfo{year}{2023}).
\newblock


\bibitem[Zhang et~al\mbox{.}(2022c)]%
        {zhang2022exploring}
\bibfield{author}{\bibinfo{person}{YiFan Zhang}, \bibinfo{person}{Hanlin
  Zhang}, \bibinfo{person}{Zachary~Chase Lipton}, \bibinfo{person}{Li~Erran
  Li}, {and} \bibinfo{person}{Eric Xing}.} \bibinfo{year}{2022}\natexlab{c}.
\newblock \showarticletitle{Exploring transformer backbones for heterogeneous
  treatment effect estimation}. In \bibinfo{booktitle}{\emph{NeurIPS ML Safety
  Workshop}}.
\newblock


\bibitem[Zhang et~al\mbox{.}(2023a)]%
        {zhang2023adanpc}
\bibfield{author}{\bibinfo{person}{Yi-Fan Zhang}, \bibinfo{person}{Xue Wang},
  \bibinfo{person}{Kexin Jin}, \bibinfo{person}{Kun Yuan},
  \bibinfo{person}{Zhang Zhang}, \bibinfo{person}{Liang Wang},
  \bibinfo{person}{Rong Jin}, {and} \bibinfo{person}{Tieniu Tan}.}
  \bibinfo{year}{2023}\natexlab{a}.
\newblock \showarticletitle{AdaNPC: Exploring Non-Parametric Classifier for
  Test-Time Adaptation}.
\newblock \bibinfo{journal}{\emph{ICML}} (\bibinfo{year}{2023}).
\newblock


\bibitem[Zhang et~al\mbox{.}(2022b)]%
        {zhang2021learning}
\bibfield{author}{\bibinfo{person}{Yi-Fan Zhang}, \bibinfo{person}{Zhang
  Zhang}, \bibinfo{person}{Da Li}, \bibinfo{person}{Zhen Jia},
  \bibinfo{person}{Liang Wang}, {and} \bibinfo{person}{Tieniu Tan}.}
  \bibinfo{year}{2022}\natexlab{b}.
\newblock \showarticletitle{Learning domain invariant representations for
  generalizable person re-identification}.
\newblock \bibinfo{journal}{\emph{IEEE Transactions on Image Processing}}
  (\bibinfo{year}{2022}).
\newblock


\bibitem[Zhao et~al\mbox{.}(2019)]%
        {zhao2019learning}
\bibfield{author}{\bibinfo{person}{Han Zhao}, \bibinfo{person}{Remi~Tachet
  Des~Combes}, \bibinfo{person}{Kun Zhang}, {and} \bibinfo{person}{Geoffrey
  Gordon}.} \bibinfo{year}{2019}\natexlab{}.
\newblock \showarticletitle{On learning invariant representations for domain
  adaptation}. In \bibinfo{booktitle}{\emph{ICML}}. PMLR.
\newblock


\end{thebibliography}

\appendix
\clearpage
\newpage

\section{Proofs of Theoretical Statements}\label{app:proofs}
To complete the proofs, we begin by introducing some necessary definitions. 

\begin{defn}
($\mathcal{H}$-divergence~\cite{2006Analysis}). Given two domain distributions $\mathcal{D}_S,\mathcal{D}_\T$ over $X$, and a hypothesis class $\mathcal{H}$, the \textit{$\mathcal{H}$-divergence} between $\mathcal{D}_S,\mathcal{D}_\T$ is $ d_{\mathcal{H}}(\mathcal{D}_S,\mathcal{D}_\T)=2\sup_{f\in\mathcal{H}}\mid\mathbb{E}_{x\sim\mathcal{D}_S}[f(x)=1]-\mathbb{E}_{x\sim\mathcal{D}_\T}[f(x)=1]\mid$.
\label{define1}
\end{defn}

\subsection{Derivation and Explanation of the Learning Bound in \myref{lemma:bound}}
\label{app:bound}


Let $f^*=\arg\min_{\hat{f}\in\mathcal{H}}\left(\epsilon_\T(\hat{f})+\sum_{i=1}^K\epsilon_i(\hat{f})\right)$, and let $\lambda_\T$ and $\lambda_i$ be the errors of $f^*$ with respect to $\mathcal{D}_\T$ and $\mathcal{D}_i$ respectively. Notice that $\lambda_\alpha=\lambda_\T+\sum_{i=1}^K\lambda_i$.  Similar to~\cite{2006Analysis} (Theorem 1), we have
\begin{equation}
\begin{aligned}
    \epsilon_\mathcal{T}(\hat{f})&\leq\lambda_\T+P_{\mathcal{D}_\T}[\mathcal{Z}_h\triangle\mathcal{Z}_h^*] \\ & \leq  \lambda_\T+P_{\mathcal{D}_\alpha}[\mathcal{Z}_h\triangle\mathcal{Z}_h^*]+|P_{\mathcal{D}_\alpha}[\mathcal{Z}_h\triangle\mathcal{Z}_h^*]-P_{\mathcal{D}_\T}[\mathcal{Z}_h\triangle\mathcal{Z}_h^*]|\\
    & \leq \lambda_\T+P_{\mathcal{D}_\alpha}[\mathcal{Z}_h\triangle\mathcal{Z}_h^*]+d_{\mathcal{H}}(\tilde{D}_\T,\tilde{\mathcal{D}}_\alpha)\\
    & \leq \lambda_\T+P_{\mathcal{D}_\alpha}[\mathcal{Z}_h\triangle\mathcal{Z}_h^*]+d_{\mathcal{H}}(\tilde{\mathcal{D}}^\alpha_\T,\tilde{\mathcal{D}}_\alpha)+d_{\mathcal{H}}(\tilde{\mathcal{D}}_\T,\tilde{\mathcal{D}}^\alpha_\T)\\
    & \leq \lambda_\T+\sum_{i=1}^K\lambda_i+{\sum_{i=1}^{K}\alpha_i\epsilon_{i}(\hat{f})}+{d_{\mathcal{H}}(\tilde{\mathcal{D}}^\alpha_\T,\tilde{\mathcal{D}}_\alpha)} +{d_{\mathcal{H}}(\tilde{\mathcal{D}}_\T,\tilde{\mathcal{D}}^\alpha_\T)}\\
    & \leq \lambda_\alpha+{\sum_{i=1}^{K}\alpha_i\epsilon_{i}(\hat{f})}+{d_{\mathcal{H}}(\tilde{\mathcal{D}}^\alpha_\T,\tilde{\mathcal{D}}_\alpha)} +{d_{\mathcal{H}}(\tilde{\mathcal{D}}_\T,\tilde{\mathcal{D}}^\alpha_\T)},
\end{aligned}
\label{proof:boundlabel}
\end{equation}
The forth inequality holds because of the triangle inequality. We provide the explanation for our bound in \myref{proof:boundlabel}. The second term is the empirical loss for the convex combination of all source domains. The third term corresponds to ``To what extent can the convex combination of the source domain approximate the target domain''. The minimization of the third term requires diverse data or strong data augmentation such that the unseen distribution lies within the convex combination of source domains. For the fourth term, the following equation holds for any two distributions $D_\T',D_\T''$, which are the convex combinations of source domains \cite{albuquerque2020adversarial}
\begin{equation}
    \begin{aligned}
        d_{\mathcal{H}[D_\T',D_\T'']}\leq \sum_{l=1}^{K}\sum_{k=1}^{K} \alpha_l\alpha_kd_{\mathcal{H}}[\mathcal{D}_{l},\mathcal{D}_{k}]
    \end{aligned}
\end{equation}
The upper bound will be minimized when $d_{\mathcal{H}}[\mathcal{D}_{l},\mathcal{D}_{k}]=0,\forall \, l,k\in\{1,...,K\}$. That is, projecting the source domain data into a feature space where the source domain labels are hard to distinguish.

\subsection{Derivation the Learning Bound in \myref{theo:2bound}}
\label{app:2bound}
\begin{prop}
Let $\{{\mathcal{D}_i},f_i\}_{i=1}^K$ and ${\mathcal{D}}_\T,f_\T$ be the empirical distributions and the corresponding labeling function. For any hypothesis $\hat{f}\in\mathcal{H}$, given mixed weights $\{\alpha_i\}_{i=1}^K;\sum_{i=1}^K\alpha_i=1,\alpha_i\geq0$, we have:
$$\epsilon_\T(\hat{f})\leq \sum_{i=1}^K\left(\mathbb{E}_{X\sim{\mathcal{D}}_i}\left[\alpha_i\frac{P_{\T}(X)}{P_{i}(X)}|\hat{f}-f_i|\right]+\alpha_i\mathbb{E}_{{\mathcal{D}}_\T}[|f_i-f_\T|]\right)$$
\end{prop}
\begin{proof}
\begin{equation*}
\begin{aligned}
\epsilon_\T(\hat{f})&=\epsilon_\T(\hat{f},f_\T)=\mathbb{E}_{X\sim {\mathcal{D}}_\T}[|\hat{f}(X)-f_\T(X)|]\\
&=\sum_{i=1}^K\alpha_i\mathbb{E}_{X\sim {\mathcal{D}}_\T}[|\hat{f}(X)-f_\T(X)|]\\
&= \sum_{i=1}^K\alpha_i\left(\mathbb{E}_{X\sim {\mathcal{D}}_\T}[|\hat{f}(X)-f_i(X)+f_i(X)-f_\T(X)|]\right)\\
&\leq \sum_{i=1}^K\alpha_i\left(\mathbb{E}_{X\sim {\mathcal{D}}_\T}[|\hat{f}(X)-f_i(X)|]+\mathbb{E}_{X\sim {\mathcal{D}}_\T}[|f_i(X)-f_\T(X)|]\right)
\end{aligned}
\end{equation*}
The above proof is based on absolute value inequality. After that, we ignore $X$ in the hypothesis $f(X)\rightarrow f$ for simplicity and apply the change-of-measure trick.
\begin{equation*}
\begin{aligned}
\epsilon_\T(\hat{f})&\leq\sum_{i=1}^K\alpha_i\left(\mathbb{E}_{{\mathcal{D}}_\T}[|\hat{f}-f_i|]+\mathbb{E}_{{\mathcal{D}}_\T}[|f_i-f_\T|]\right)\\
&=\sum_{i=1}^K\alpha_i\left(\int|\hat{f}-f_i|P_{\T}(X)d_X+\mathbb{E}_{{\mathcal{D}}_\T}[|f_i-f_\T|]\right)\\
&=\sum_{i=1}^K\alpha_i\left(\int|\hat{f}-f_i|P_{i}(X)\frac{P_{\T}(X)}{P_{i}(X)}d_X+\mathbb{E}_{{\mathcal{D}}_\T}[|f_i-f_\T|]\right)\\
&=\sum_{i=1}^K\left(\mathbb{E}_{X\sim{\mathcal{D}}_i}\left[\alpha_i\frac{P_{\T}(X)}{P_{i}(X)}|\hat{f}-f_i|\right]+\alpha_i\mathbb{E}_{{\mathcal{D}}_\T}[|f_i-f_\T|]\right),
\end{aligned}
\end{equation*}
which completes our proof.
\end{proof}

\subsection{Comparison of the proposed bound to existing bound.}\label{sec:app_tight}
Before we derive our main result, we first introduce some necessary theorems. For simplicity, given hypothesis $\hat{f},\hat{f}'\in\mathcal{H}$ and label function $f_S$ for $\mathcal{D}_S$, denote $\epsilon_S(\hat{f},\hat{f}')=\mathbb{E}_{{\tilde{\mathcal{D}}}_S}[|\hat{f}-\hat{f}'|]$ and $\epsilon_S(\hat{f})=\epsilon_S(\hat{f},f_S)=\mathbb{E}_{\tilde{\mathcal{D}}_S}[|\hat{f}-f_S|]$, we have
\begin{theo}
(Lemma 4.1 and Theorem 4.1 in~\cite{albuquerque2019generalizing}.) Given twodistributions in the image space $<{\mathcal{D}}_S,f_S>,<{\mathcal{D}}_\T,f_\T>$ and $\hat{f}\in\mathcal{H}$, we have
\begin{equation}
\left|\epsilon_S(f_S,f_\T)-\epsilon_\T(f_S,f_\T)\right|\leq d_\mathcal{H}({\mathcal{D}}_S, {\mathcal{D}}_\T).
\label{equ:inequ_dh}
\end{equation}
The error in the target domain can then be bounded by 
\begin{equation}
    \epsilon_\T(\hat{f})\leq\epsilon_S(\hat{f})+d_\mathcal{H}({\mathcal{D}}_S, {\mathcal{D}}_\T)+\min\{\epsilon_S(f_S,f_\T),\epsilon_\T(f_S,f_T)\},
\label{equ:bound_con}
\end{equation}
where the result is based mainly on the inequality in \myref{equ:inequ_dh}.
\label{theo:error_da}
\end{theo}

If only two domains are considered, namely, given $<{\mathcal{D}}_S,f_S>,<{\mathcal{D}}_\T,f_\T>$, recall the derivation of the proposed error bound; we have
\begin{equation}
\begin{aligned}
\epsilon_\T(\hat{f})
&\leq\mathbb{E}_{{{\mathcal{D}}}_\T}[|h-f_S|]+\mathbb{E}_{{{\mathcal{D}}}_\T}[|f_S-f_\T|]\\
&=\epsilon_\T(\hat{f},f_S)+\epsilon_\T(f_S,f_\T)\\
&=\mathbb{E}_{X\sim{{\mathcal{D}}}_S}\left[\frac{P_{\T}(X)}{P_{S}(X)}|\hat{f}-f_S|\right]+\epsilon_\T(f_S,f_\T)
\end{aligned}
\label{equ:bound_ours}
\end{equation}
Then we will prove that \myref{equ:bound_ours} is upper bounded by \myref{equ:bound_con}. At first, the second line in \myref{equ:bound_ours} is bounded by
\begin{equation}
\begin{aligned}
\epsilon_\T(\hat{f},f_S)+\epsilon_\T(f_S,f_\T)
&\leq \epsilon_S(\hat{f},f_S)+d_\mathcal{H}({\mathcal{D}}_S, {\mathcal{D}}_\T)+\epsilon_\T(f_S,f_\T)\\
&=\epsilon_S(\hat{f})+\epsilon_\T(f_S,f_\T)+d_\mathcal{H}({\mathcal{D}}_S, {\mathcal{D}}_\T).
\end{aligned}
\label{equ:proof_thigt1}
\end{equation}
Also, since the density ratio $\frac{P_\T(X)}{P_S(X)}$ is intractable and during implementation, this term is set to a constant and ignored. That is, the last line of \myref{equ:bound_ours} is approximately equal to 
\begin{equation}
\begin{aligned}
&\mathbb{E}_{X\sim{{\mathcal{D}}}_S}\left[\frac{P_{\T}(X)}{P_{S}(X)}|\hat{f}-f_S|\right]+\epsilon_\T(f_S,f_\T)\\
=&\epsilon_S(\hat{f})+\epsilon_\T(f_S,f_\T)\\
\leq& \epsilon_S(\hat{f})+\epsilon_S(f_S,f_\T)+d_\mathcal{H}({\mathcal{D}}_S, {\mathcal{D}}_\T)
\end{aligned}
\label{equ:proof_thigt2}
\end{equation}
Combining \myref{equ:proof_thigt1} and \myref{equ:proof_thigt2} we can get the error bound in \myref{equ:bound_ours} is upper bounded by $\epsilon_S(\hat{f})+d_\mathcal{H}({\mathcal{D}}_S, {\mathcal{D}}_\T)+\min\{\epsilon_S(f_S,f_\T),\epsilon_\T(f_S,f_\T)\}$, which completes our proof.

\subsection{Reformulation of the density ratio.}\label{sec:app_density}

In this subsection, we first introduce some important definitions of the distributionally robust optimization (DRO) framework \cite{ben2009robust} and then reformulate the density ratio under some necessary assumptions. In DRO, the expected worst-case risk on a predefined family of distributions $\mathcal{Q}$ (termed \textit{uncertainty set}) is used to replace the expected risk on the unseen target distribution $\mathcal{T}$ in ERM. Therefore, the objective is as follows. 
\begin{equation}
    \min_{\theta\in\Theta}\max_{q\in\mathcal{Q}}\mathbb{E}_{(x,y)\in q}[\ell(x,y;\theta)].
    \label{equ_dro}
\end{equation} 
Specifically, the uncertainty set $\mathcal{Q}$ encodes the possible test distributions on which we want our model to perform well. If $\mathcal{Q}$ contains $\mathcal{T}$, the DRO object can upper bound the expected risk under $\mathcal{T}$. 

The construction of uncertainty set $\mathcal{Q}$ is of vital importance. Here we reformulate the density ratio based on the KL-divergence ball constraint and other choices (\eg using the moment constraint~\cite{delage2010distributionally}, $f$-divergence~\cite{michel2021modeling}, Wasserstein/MMD ball~\cite{sinha2017certifying,staib2019distributionally}) will lead to different reweighting methods. Given the KL upper bound (radius) $\eta$, denote the empirical distribution $\mathcal{P}$, we have the uncertainty set  $\mathcal{Q}=\{Q:\text{KL}(Q||\mathcal{P})\leq\eta\}$. 
The Min-Max Problem in~\myref{equ_dro} can then be reformulated as
\begin{equation}
    \min_{\theta\in\Theta}\max_{Q:\text{KL}(Q||\mathcal{P})\leq\eta}\mathbb{E}_{(x,y)\in Q}\left[\ell(x,y;\theta)\right].
    \label{equ_kldro}
\end{equation}
Then we have the following theorem, which derives the optimal density ratio and converts the original problem to a reweighting version.
\begin{theo}
(Modified from Section 2 in ~\cite{hu2013kullback}) Assume the model family $\theta\in\Theta$ and $\mathcal{Q}$ to be convex and compact. The loss $\ell$ is continuous and convex for all $x\in\mathcal{X},y\in\mathcal{Y}$. Suppose that the empirical distribution $\mathcal{P}$ has density $p(x,y)$. Then the inner maximum of~\myref{equ_kldro} has a closed-form solution: $q^*(x,y)=\frac{p(x,y)e^{\ell(x,y;\theta)/\tau^*}}{\mathbb{E}_\mathcal{P}\left[e^{\ell(x,y;\theta)/\tau^*}\right]}$, 
where $\tau^*$ satisfies $\mathbb{E}_\mathcal{P}\left[\frac{e^{\ell(x,y;\theta)/\tau^*}}{\mathbb{E}_\mathcal{P}[e^{\ell(x,y;\theta)/\tau^*}]}\left(\frac{\ell(x,y;\theta)}{\tau^*}-\log \mathbb{E}_\mathcal{P}[e^{\ell(x,y;\theta)/\tau^*}]\right)\right]=\eta$ and $q^*(x,y)$ is the optimal density of $Q$. The min-max problem in~\myref{equ_kldro} is equivalent to
\begin{equation}
    \min_{\theta\in\Theta,\tau>0} \tau\log\mathbb{E}_{\mathcal{P}}\left[e^{\ell(x,y;\theta)/\tau}\right]+\eta\tau.
    \label{equ_kldro_convex}
\end{equation}
\end{theo}

\section{Dataset and implementation details}\label{sec:data_detail}
\subsection{Dataset Details}
\textbf{Colored MNIST} \cite{arjovsky2020invariant} consists of digits in MNIST with different colors (blue or red). The label is a noisy function of the digit and color. First, a preliminary label $\bar{y}$ is assigned to images based on their digits, $\bar{y}=0$ for digits 0-4 and $\bar{y}=1$ for digits 5-9. The final label is obtained by flipping $\bar{y}$ with probability $0.25$. The color signal $z$ of each sample is obtained by flipping $y$ with probability $p^d$, where $p^d$ is $\{0.2,0.1,0.9\}$ for three different domains. Finally, images with $z=1$ will be colored red and $z=0$ will be colored blue. This dataset contains $70,000$ examples of dimension $(2,28,28)$ and $2$ classes.

\noindent\textbf{Rotated MNIST} \cite{ghifary2015domain} consists of 10,000 digits in MNIST with different rotated angles where the domain is determined by the degrees $d \in \{0, 15, 30, 45,
60, 75\}$.

\noindent\textbf{PACS} \cite{li2017deeper} includes 9, 991 images with 7 classes $y \in  \{$ dog, elephant, giraffe, guitar, horse, house, person $\}$ from 4 domains $d \in$ $\{$art, cartoons, photos, sketches$\}$. 

\noindent\textbf{VLCS} \cite{torralba2011unbiased} is composed of 10,729 images, 5 classes $y \in \{$ bird, car, chair, dog, person $\}$ from domains $d \in \{$Caltech101, LabelMe, SUN09, VOC2007$\}$. 

\noindent\textbf{DomainNet}~\cite{peng2019moment} has six domains $d \in$ $\{$clipart, infograph, painting, quickdraw, real, sketch$\}$. This dataset contains $586,575$ examples of size $(3,224,224)$ and $345$ classes.

\end{document}